%% file: Arxiv_main.tex
\documentclass[11pt,twoside]{article}

\usepackage{fullpage}

\usepackage{epsf}
\usepackage{fancyhdr}
\usepackage{graphics}
\usepackage{graphicx}
\usepackage{psfrag}
\usepackage{comment}

\usepackage[linesnumbered,ruled]{algorithm2e}
\DontPrintSemicolon	

\usepackage{color}

\usepackage[utf8]{inputenc}

\usepackage{amsthm}
\usepackage{amsfonts}
\usepackage{amsmath}
\usepackage{amssymb,bbm}
\input{Arxiv_final_macros}

\usepackage{url}
\usepackage[colorlinks=True,linkcolor=magenta,citecolor=blue,urlcolor=blue,pagebackref=true,backref=true]
{hyperref}
\renewcommand*{\backref}[1]{\ifx#1\relax \else Page #1 \fi}
\renewcommand*{\backrefalt}[4]{%
    \ifcase #1 \footnotesize{(Not cited.)}%
    \or        \footnotesize{(Cited on page~#2.)}%
    \else      \footnotesize{(Cited on pages~#2.)}%
    \fi}

\usepackage{nicefrac}

\usepackage{chngpage}

 \usepackage{tabularx}%

\usepackage{enumitem}
\usepackage{booktabs}
\usepackage{pbox}

\usepackage{caption,subcaption}

\usepackage{etoc}
\usepackage{mathtools}

\usepackage{fullpage}

\allowdisplaybreaks

\makeatletter
\long\def\@makecaption#1#2{
        \vskip 0.8ex
        \setbox\@tempboxa\hbox{\small {\bf #1:} #2}
        \parindent 1.5em  
        \dimen0=\hsize
        \advance\dimen0 by -3em
        \ifdim \wd\@tempboxa >\dimen0
                \hbox to \hsize{
                        \parindent 0em
                        \hfil 
                        \parbox{\dimen0}{\def\baselinestretch{0.96}\small
                                {\bf #1.} #2
                                } 
                        \hfil}
        \else \hbox to \hsize{\hfil \box\@tempboxa \hfil}
        \fi
        }
\makeatother

\begin{document}

\etocdepthtag.toc{mtchapter}
\etocsettagdepth{mtchapter}{subsection}
\etocsettagdepth{mtappendix}{none}


\begin{center}
{\bf{\Large{On the Minimax Optimality of the EM Algorithm for Learning Two-Component Mixed Linear Regression}}}

\vspace*{.2in}
 {\large{
 \begin{tabular}{ccc}
  Jeong Yeol Kwon$^{\diamond}$ & Nhat Ho$^{\dagger}$ &  Constantine Caramanis$^{\diamond}$ \\
 \end{tabular}

}}

\vspace*{.2in}

 \begin{tabular}{c}
 Department of Electrical and Computer Engineering $^\diamond$\\
 Depatment of Statistics and Data Sciences$^\dagger$ \\
 University of Texas, Austin
 \end{tabular}

\vspace*{.2in}

\today

\vspace*{.2in}

\begin{abstract}
  We study the convergence rates of the EM algorithm for learning two-component mixed linear regression under all regimes of signal-to-noise ratio (SNR). We resolve a long-standing question that many recent results have attempted to tackle: we completely characterize the convergence behavior of EM, and show that the EM algorithm achieves minimax optimal sample complexity under all SNR regimes. In particular, when the SNR is sufficiently large, the EM updates converge to the true parameter $\theta^{*}$ at the standard parametric convergence rate $\calo((d/n)^{1/2})$ after $\calo(\log(n/d))$ iterations. In the regime where the SNR is above $\calo((d/n)^{1/4})$ and below some constant, the EM iterates converge to a $\calo({\rm SNR}^{-1} (d/n)^{1/2})$ neighborhood of the true parameter, when the number of iterations is of the order $\calo({\rm SNR}^{-2} \log(n/d))$. In the low SNR regime where the SNR is below $\calo((d/n)^{1/4})$, we show that EM converges to a $\calo((d/n)^{1/4})$ neighborhood of the true parameters, after $\calo((n/d)^{1/2})$ iterations. Notably, these results are achieved under mild conditions of either random initialization or an efficiently computable local initialization. By providing tight convergence guarantees of the EM algorithm in middle-to-low SNR regimes, we fill the remaining gap in the literature, and significantly, reveal that in low SNR, EM changes rate, matching the $n^{-1/4}$ rate of the MLE, a behavior that previous work had been unable to show.

\end{abstract}
\end{center}

\input{Arxiv_introduction}

\input{Arxiv_main_results}

\subsection{Towards unknown variance and weight}
\label{subsection:unknown_extension}
In this section, we discuss the statistical behavior of the EM algorithm when either the variance $\sigma^{*} = 1$ or the mixing weight of the true density $g_{true}$ is unknown.
\paragraph{Unknown noise variance:} We first discuss the case when the variance $\sigma^{*}$ of regression noise is unknown. In this case, the EM updates for $\theta$ and $\sigma$ are as follows:
\begin{align}
    \bar{\theta}_{n}^{t + 1} & = \parenth{ \frac{1}{n} \sum_{i = 1}^{n} X_{i} X_{i}^{\top}}^{-1} \times \parenth{\frac{1}{n} \sum_{i = 1}^{n} \tanh \parenth{\frac{Y_{i} X_i^\top \theta_n^t}{(\bar{\sigma}_n^t)^2} } Y_i X_i }, \nonumber \\
    (\bar{\sigma}_n^{t+1})^2 & = \frac{1}{n} \sum_{i = 1}^{n} Y_{i}^2 - \bigr((\bar{\theta}_{n}^{t + 1})^{\top} X_{i} \bigr)^2. \label{eq:EM_updates_mixed_regression_unknown_noise}
\end{align}
The EM update $\bar{\theta}_{n}^{t + 1}$ in the unknown variance case depends on $\sigma_{n}^{t}$, which is updated at each iteration of the EM algorithm. It is different from the update for $\theta$ in the known variance case in equation~\eqref{eq:EM_updates_mixed_regression_known_noise}. Therefore, the overall analysis of the EM algorithm in the unknown variance case should be re-derived in the population level to get right contraction coefficients of the population EM updates. 

We would like to remark that the challenge with more unknowns arises from the convergence analysis in the population level, and the sample complexity analysis is irrelevant to whether we have more unknowns or not. In this section, we provide the statistical behavior of the EM algorithm under the low SNR regime and leave the complete analysis of the EM algorithm for future work. The localization technique used in the low SNR regime of known variance setting remains to be useful for obtaining the convergence and statistical rates of EM in the low SNR regime of unknown variance setting. It leads to the following result with the EM iterates in the low SNR regime.
\begin{theorem} \label{theorem:algebraic_independent}
(Low SNR regime of unknown variance case) There exist universal constants $C_0, C_1, C_2, C_3 > 0$ such that when $\enorm{\thetastar} \le C_0 (d \log^2(n/\delta)/n)^{1/4}$, starting from $\|\bar{\theta}_n^0\| \le 0.2$ and $|(\bar{\sigma}_n^0)^2 - 1| \le 0.04$, the EM updates \eqref{eq:EM_updates_mixed_regression_unknown_noise} return $(\bar{\theta}_{n}^{t}, \bar{\sigma}_{n}^{t})$ which satisfies
\begin{align*}
    \|\bar{\theta}_{n}^{t} - \theta^{*}\| \le C_1(d \log^2(n/\delta) /n)^{1/ 4}, \\
    | (\bar{\sigma}_{n}^{t})^2 - (\sigma^{*})^2 | \le C_2 (d \log^2(n/\delta)/n)^{1/2},
\end{align*}
with probability at least $1 - \delta$ after $t \ge C_3 \log (\log(n/d)) \sqrt{n/ (d \log^2(n/\delta))}$ iterations.
\end{theorem}
Unlike in the case of Gaussian mixtures with unknown variances~\cite{Raaz_Ho_Koulik_2018_second}, the statistical rate of EM updates for $\theta$ is $(d/n)^{1/4}$ for all $d \geq 1$. When $\theta^{*} = 0$, this coincides with the previous result on the rate of maximum likelihood estimation for over-specified Gaussian mixture of experts~\cite{ho2019rate}, where it is shown that the MLE rate of estimating $\theta^{*}$ is $n^{-1/4}$ as long as the link functions are algebraically independent, which is the case for the unknown variance setting, and the number of components of Gaussian mixtures of experts is over-specified. The proof of Theorem~\ref{theorem:algebraic_independent} is given in Appendix~\ref{appendix:proof:contraction_ind}.

\paragraph{Unknown mixing weights:} The extension to the unknown mixing weight can be more challenging, since the unbalanced mixing weight induces asymmetry in the landscape of the log-likelihood function. The asymmetry completely changes the population landscape of two-component mixed linear regression (e.g., there is a local maxima in the population log-likelihood for a mixture of two Gaussian distributions, which is absent in the symmetric setting~\cite{xu2018benefits}). It makes the analysis of the EM algorithm challenging even in the two-component settings of mixed linear regression. In high SNR regimes, we can avoid direct analysis of the optimization landscape and still can show the linear convergence of the EM iterates toward true parameters~\cite{kwon2020converges}. However, in middle-to-low SNR regimes, we cannot avoid the analysis of complicated landscape. The extension to unknown mixing weights is an interesting future direction.

\input{Arxiv_proof_sketch}

\section{Conclusion}
\label{sec:conclusion}
In the paper, we completely characterize the convergence behavior of EM under all SNR regimes of symmetric two-component mixed linear regression. We view our results for this model as the first step towards a comprehensive understanding of the EM algorithm for learning weakly separated latent variable models. We now discuss a few future directions naturally arise from our work. First, in more general settings of weakly separated mixture models with $k$ components, it is known that the rate of MLE can be $n^{-O(1/k)}$ in the worst case~\cite{Jonas-2016}. Furthermore, EM is known to suffer from very slow convergence in practice for instances with large overlaps. It is an important future direction to characterize the convergence behavior of the EM algorithm in such settings. Second, our results demonstrate that the EM algorithm has sub-linear convergence to $\theta^{*}$ under middle and low SNR regimes. It respectively leads to $\enorm{\theta^{*}}^{-2} \log(n/ d)$ and $\sqrt{n/d}$ number of iterations under middle-to-low SNR regimes, which result in high computational complexity. An important direction is to develop an alternative to EM algorithm that can achieve much cheaper computational complexity and also obtain minimax optimal sample complexity under all SNR regimes of mixed linear regression. Finally, while we prove that the EM algorithm achieves minimax optimal statistical convergence rates for learning two-component mixed linear regression, it is important to further develop uncertainty quantification for the EM iteratates, such as confidence intervals. It necessitates the future study on the central limit theorem of the EM algorithm under all regimes of SNR, which has remained a major open problem in the literature.

\bibliography{Arxiv_main}

\begin{appendix}
\input{Arxiv_appendix}
\end{appendix}

\end{document}

%% file: Arxiv_final_macros.tex
\newcommand{\event}{\mathcal{E}}

\newcommand{\calo}{\ensuremath{\mathcal{O}}}

\newcommand{\thetastar}{\ensuremath{\theta^*}}
\newcommand{\sdstar}{\ensuremath{\sigma^*}}

\newcommand{\NORMAL}{\ensuremath{\mathcal{N}}}

\newcommand{\brackets}[1]{\left[ #1 \right]}
\newcommand{\parenth}[1]{\left( #1 \right)}

\newcommand{\abss}[1]{\left| #1 \right |}

\newcommand{\Rspace}{\ensuremath{\mathbb{R}}}

\newcommand{\ball}{\ensuremath{\mathbb{B}}}
\newcommand{\sphere}{\ensuremath{\mathbb{S}}}

\newcommand{\samopmlr}{\ensuremath{M}_{n,\text{mlr}}}
\newcommand{\popopmlr}{\ensuremath{M}_{\text{mlr}}}
\newcommand{\samopind}{\ensuremath{M}_{n,\text{ind}}}
\newcommand{\popopind}{\ensuremath{M}_{\text{ind}}}

\newcommand{\matsnorm}[2]{|\!|\!| #1 | \! | \!|_{{#2}}}
\newcommand{\vecnorm}[2]{\left\| #1\right\|_{#2}}

\newcommand{\enorm}[1]{\ensuremath{\| #1 \|}} 

\newcommand{\opnorm}[1]{\ensuremath{\matsnorm{#1}{\tiny{\mbox{op}}}}}

\newcommand{\Exs}{\ensuremath{{\mathbb{E}}}}
\newcommand{\Prob}{\ensuremath{{\mathbb{P}}}}

\newtheoremstyle{named}{}{}{\itshape}{}{\bfseries}{.}{.5em}{\thmnote{#3's }#1}
\theoremstyle{named}

\theoremstyle{plain}

\newtheorem{theorem}{Theorem}
\newtheorem{remark}{Remark}

\newtheorem{lemma}{Lemma}

\newtheorem{corollary}{Corollary}

\theoremstyle{definition}

\newlength{\widebarargwidth}

\makeatletter
\long\def\@makecaption#1#2{
        \vskip 0.8ex
        \setbox\@tempboxa\hbox{\small {\bf #1:} #2}
        \parindent 1.5em  
        \dimen0=\hsize
        \advance\dimen0 by -3em
        \ifdim \wd\@tempboxa >\dimen0
                \hbox to \hsize{
                        \parindent 0em
                        \hfil
                        \parbox{\dimen0}{\def\baselinestretch{0.96}\small
                                {\bf #1.} #2
                                }
                        \hfil}
        \else \hbox to \hsize{\hfil \box\@tempboxa \hfil}
        \fi
        }
\makeatother

\long\def\comment#1{}
\definecolor{carnelian}{rgb}{0.7, 0.11, 0.11}
\definecolor{battleshipgrey}{rgb}{0.52, 0.52, 0.51}
\definecolor{darkgray}{rgb}{0.66, 0.66, 0.66}
\definecolor{indiagreen}{rgb}{0.07, 0.53, 0.03}
\definecolor{darkgreen}{rgb}{0.0, 0.2, 0.13}
\definecolor{darkspringgreen}{rgb}{0.09, 0.45, 0.27}
\definecolor{dukeblue}{rgb}{0.0, 0.0, 0.61}
\definecolor{olivedrab7}{rgb}{0.24, 0.2, 0.12}
\definecolor{darkblue}{rgb}{0.0, 0.0, 0.55}
\definecolor{darkscarlet}{rgb}{0.34, 0.01, 0.1}
\definecolor{candyapplered}{rgb}{1.0, 0.03, 0.0}
\definecolor{ao(english)}{rgb}{0.0, 0.5, 0.0}
\definecolor{applegreen}{rgb}{0.55, 0.71, 0.0}



%% file: Arxiv_introduction.tex
\section{Introduction} 
\label{sec:introduction}
The expectation-maximization (EM) algorithm is a general-purpose heuristic to compute a maximum-likelihood estimator (MLE) for problems with missing information \cite{Rubin-1977, Jeff_Wu-1983, redner1984mixture}. In general, computing the MLE is intractable due to the non-concave nature of log-likelihood functions in the presence of missing data. The EM algorithm iteratively computes a tighter lower bound on log-likelihood functions, with each iteration no more complex than solving a maximum-likelihood (ML) problem without missing data. Due to its simplicity and broad success in practice, EM is one of the most popular methods-of-choice in a variety of applications \cite{Xu_Jordan-1995, Jordan-2000, Chen_2008_biometrika, Chen_2009}. 

Recent years have witnessed remarkable progress in establishing theory describing the non-asymptotic convergence of EM  to the true parameters on canonical examples such as a mixture of Gaussian distributions and mixed linear regression (see Prior Art below). In such models, a key factor in the analysis is the separation between components, or the ``signal strength''. Most prior work has studied strongly separated instances (high SNR) and established linear convergence of the EM algorithm with the standard parametric statistical rate $n^{-1/2}$. In contrast, the understanding of the EM algorithm in the weakly separated settings (low SNR), especially mixed linear regression, remains incomplete.

\textbf{Our contributions:} In this paper, we aim to fill the remaining gap in the literature with the minimax optimal sample complexity of the EM algorithm for learning two-component mixed linear regression in the weakly separated regime. In so doing, we provide a complete picture of the EM algorithm under all signal-to-noise ratio (SNR) regimes for symmetric two-component mixed linear regression, namely, $\frac{1}{2} \NORMAL(-X^\top \thetastar, (\sdstar)^2) + \frac{1}{2} \NORMAL(X^\top \thetastar, (\sdstar)^2)$ where $\sigma^{*} = 1$ is given and $X$ follows the standard multivariate normal distribution in $d$ dimensions. We define SNR as $\eta := \enorm{\theta^*}$ since $\sigma^* = 1$. Notably, our results are obtained under mild conditions of either random initialization or an efficiently computable local initialization. While simplified, the model is complex enough to capture the most interesting behaviors of the EM algorithm for learning a mixed linear regression with two components, and reveals statistical behaviors in the low-to-middle SNR regimes that previous analysis had missed. In summary, our contributions are as follows.
\begin{enumerate}
    \item \textbf{High-to-middle SNR regimes}: when $(d/n)^{1/4} \lesssim \enorm{\theta^*}$ (up to some logarithmic factor), the EM updates converges to $\thetastar$ within a neighborhood of $\calo(\max\{1, \enorm{\theta^*}^{-1}\} (d/n)^{1/2})$ after $\calo(\max \{1, \enorm{\theta^*}^{-2}\} \log(n/d))$ number of iterations. 
    \item \textbf{Low SNR regime}: when $\enorm{\thetastar} \lesssim (d/n)^{1/4}$ (up to some logarithmic factor), the EM algorithm converge to $\thetastar$ within a neighborhood of $\calo((d/n)^{1/4})$ when the number of iterations is of the order of $\calo((n/d)^{1/2})$.
    \item \textbf{Global Convergence}: We demonstrate that EM converges from {\it any} randomly initialized point with high probability. Furthermore, we do not require sample-splitting in our analysis.
\end{enumerate}

While we discuss the tightness of our result in a great detail in Section~\ref{subsection:tightness}, we briefly explain the significance of our results. We focus primarily on two aspects of the EM algorithm: (i) statistical rate, and (ii) computational complexity. In the high SNR regime, we have linear convergence to true parameters within $\sqrt{d/n}$ rate as noted previously in the literature. In contrast, in the low SNR regime when $\enorm{\theta^*} \lesssim (d/n)^{1/4}$, the statistical rate is $(d/n)^{1/4}$. We explain this transition in statistical rate with a convergence property of the population EM in the middle-to-low SNR regimes. The upper bound given by EM matches the known lower bound for this problem in all SNR regimes \cite{chen2014convex}. For the computational complexity, the number of iterations increases quadratically in the inverse of SNR until SNR reaches $(d/n)^{1/4}$. Interestingly, the number of iterations is naturally interpolated at $\text{SNR} = (d/n)^{1/4}$ from $\enorm{\theta^*}^{-2} \log(n/d)$ to $\sqrt{n/d}$. More in-depth discussions on the results (e.g., detailed comparison to previous works, proof techniques we use, etc.) are provided in Section~\ref{subsection:tightness}.


\subsection{Prior Art}
While the classical results on the EM algorithm only guaranteed asymptotic convergence to {\it stationary points} \cite{Jeff_Wu-1983}, the seminal work~\cite{Siva_2017} proposed a general framework to study a non-asymptotic convergence of the EM algorithm to {\it true parameters}. Motivated by this work, there has been a flurry of work studying the convergence of the EM algorithm to the true parameters for various kinds of regular mixture models (see e.g., \cite{yi2014alternating, yi2016solving, Hsu-nips2016, Sarkar_nips2017, Daskalakis_colt2017, kwon2020converges, Raaz_Ho_Koulik_2018_second, kwon2020algorithm}). Most of the work in this line require strong separation compared to the noise level, i.e., considers the high SNR regime. Using this condition, it establishes linear convergence of EM to parameter estimates that lie within $(d/n)^{1/2}$-radius around the true location parameters. In contrast, relatively little understanding is available when different components in a mixture model are weakly separated (i.e., middle-to-low SNR). In particular, even for simple settings of two-component mixed linear regression that we consider in this work, our understanding on the EM algorithm still remains incomplete, for as we show, not only the techniques, but also the conclusions of past analysis no longer hold in the weakly separated regime. 

The first convergence guarantees for EM under mixed linear regression was established in a noise-free setting~\cite{yi2014alternating, yi2016solving}. Subsequent results succeeded in treating the noisy setting (see \cite{Siva_2017}) for a mixture of two linear regressions, when the the signal strength $\enorm{\theta^*}$ is significantly larger than the noise variance $\sigma^*$ (high SNR). Work in~\cite{kwon2020converges} extended the results in~\cite{Siva_2017} and~\cite{yi2016solving} to a more general setting of learning a mixture of $k$-component linear regressions when the SNR is $\Omega(k)$. However, it has not been obvious how to extend any of these results to the weakly separated regimes. 

Recently, \cite{kwon2019global} has established the global convergence of the EM algorithm for learning a mixture of two linear regressions in all SNR regimes. While their result guarantees convergence of EM in all SNR regimes, the characterization of this convergence falls short in two aspects: (i) their analysis relies on the sample-splitting, (ii) their result is sub-optimal in terms of SNR in low SNR regime. In order to elaborate more on the second aspect, the statistical rate in \cite{kwon2019global} is given as $O(\eta^{-6} n^{-1/2})$ given that the sample size $n \gtrsim \eta^{-6}$ is sufficiently large. However, it is known that in the limit setting of $\eta \rightarrow 0$, the rate of MLE slows down to $n^{-1/4}$~\cite{Chen1992, Ho-Nguyen-AOS-17, ho2019rate}. The result in \cite{kwon2019global} fails to capture this important property in relation to EM, and gives little insight on what happens when there is a large overlap between components. Our results tighten the sub-optimal analysis for middle SNR regime in~\cite{kwon2019global} and fill in the remaining gap in the literature by providing a tight convergence guarantee of the EM algorithm in low SNR regime.

In a closely related problem of learning mixtures of two Gaussians, \cite{Raaz_Ho_Koulik_2018, Raaz-misspecified, Raaz_Ho_Koulik_2018_second} recently studied an extreme case of the over-specified mixture models, {\it i.e.,} there is no separation between two components. However, their analysis is restricted to strictly over-specified settings, and it has not been obvious to extend their result to weakly-separated models. In another recent work, \cite{wu2019randomly} has studied the EM algorithm for learning a mixture of two weakly-separated location Gaussians, establishing a minimax rate of the EM algorithm after $O(\sqrt{n/ d})$ iterations in middle-to-low SNR regimes. However, their result requires the initialization to be already within a small Euclidean ball of $(d/n)^{1/4}$-radius, which is very restrictive. Our result does not suffer from small initialization issue as in \cite{wu2019randomly}. Furthermore, our proof strategy can be applied to resolve the open issue with small initialization in~\cite{wu2019randomly}. 



We note in passing that the problem of solving mixed linear regressions is an interesting problem by itself. It arises in a number of applications \cite{de1989mixtures, grun2007applications}, and has been extensively studied with various algorithms proposed (see e.g., \cite{chaganty2013spectral, chen2014convex, sedghi2016provable, yi2016solving, li2018learning, chen2019learning, karmalkar2019list,raghavendra2020list}). The special case of a mixture of two-component linear regressions is by now well understood \cite{yi2014alternating, chen2014convex, kwon2019global, ghosh2020alternating}. In this work, rather than solving a mixed linear regression itself, we focus on the rigorous study of the EM algorithm.

\textbf{Organization:} 
The remainder of our paper is organized as follows. In Section~\ref{sec:independent_setting}, we first present the setup of EM algorithm for learning symmetric two-component mixed linear regression. Then, we present the convergence rates of EM iterates under all regimes of SNR with either random initialization or computable local initialization. Finally, we discuss the tightness of the results. We present the proof sketch of the results in Section~\ref{sec:proof}. We conclude the paper in Section~\ref{sec:conclusion} while deferring the proofs of the main results in the appendices.

%% file: Arxiv_main_results.tex
\section{Convergence rates of the EM algorithm}
\label{sec:independent_setting}
We first formulate symmetric mixed linear regression with two components and EM updates for this model in Section~\ref{subsec:problem_setup}. Then, we state our main results with the convergence behaviors of EM algorithm under all regimes of SNR in Section~\ref{subsec:main_results}. Finally, we provide a detailed discussion with the tightness of the results in Section~\ref{subsection:tightness}.
\subsection{Problem setup}
\label{subsec:problem_setup}
We assume that the data $(X_{1}, Y_{1}), \ldots, (X_{n}, Y_{n})$ are generated from a symmetric two-component mixed linear regression, whose density function has the following form:
\begin{align}
    \label{eq:mixed_regression_known_noise}
    g_{\text{true}}(x, y) : = \bigr( \frac{1}{2}f(y| - (\theta^{*})^{\top} x, \sigma^{*}) + \frac{1}{2} f(y| (\theta^{*})^{\top} x, \sigma^{*}) \bigr) \bar{f}(x),
\end{align}
where $\sigma^{*} = 1$ is given and $\theta^{*}$ is an unknown parameter. Furthermore, we assume that $\bar{f}(x)$ is the density of standard multivariate Gaussian distribution, i.e., $X \sim \NORMAL(0, I_{d})$. In order to estimate $\theta^{*}$, we fit the data by using symmetric two-component mixed linear regression, which is given by:
\begin{align}
    \label{eq:em_independence_fit_known_noise}
    g_{\text{fit}}(x, y; \theta) : = \bigr( \frac{1}{2} f(y| - \theta^{\top}x, \sigma^{*}) + \frac{1}{2} f(y| \theta^{\top}x, \sigma^{*}) \bigr) \bar{f}(x).
\end{align}
It is clear that $g_{\text{fit}}(x, y; \theta^{*}) = g_{\text{true}}(x, y)$. A common approach to obtain an estimator for $\thetastar$ is by using maximum likelihood esimation (MLE). However, given that the log-likelihood function of symmetric two-component mixed linear regression is highly non-concave, the MLE does not have a closed-form expression. EM is a popular iterative algorithm to approximate the MLE. Given fitted model~\eqref{eq:em_independence_fit_known_noise}, simple algebra shows that the EM update for $\theta$ can be written as follows:
\begin{align}
    \label{eq:EM_updates_mixed_regression_known_noise}
    \theta_{n}^{t + 1} & = \parenth{ \frac{1}{n} \sum_{i = 1}^{n} X_{i} X_{i}^{\top}}^{-1} \parenth{\frac{1}{n} \sum_{i = 1}^{n} \tanh \parenth{\frac{Y_{i} X_i^\top \theta_n^t}{{\sigma^*}^2} } Y_i X_i },
\end{align}
where the hyperbolic function $\text{tanh}(x) : = (\exp(x) - \exp(-x))/(\exp(x) + \exp(-x))$ for all $x \in \Rspace$. In order to facilitate the ensuing argument, let us the denote population and finite-sample EM operators by Eqns. \ref{eq:sample_operator_mlr} and \ref{eq:population_operator_mlr}, respectively, as given below:
\begin{align}
    M_{mlr} (\theta) & : = \Exs[XY\tanh(YX^\top {\theta})], \label{eq:sample_operator_mlr} \\
    M_{n, mlr} (\theta) & : = \parenth{\frac{1}{n} \sum_i X_i X_i^\top}^{-1} \parenth{\frac{1}{n} \sum_i X_i Y_i \tanh(Y_i X_i^\top {\theta})}. \label{eq:population_operator_mlr}
\end{align}
\paragraph{Motivation from experiments:} In Figure~\ref{figure:em_behavior}, we present the statistical rate and optimization complexity of EM algorithm under different regimes of SNR. We set $d=5$ and initialized the estimator in the neighborhood of the true parameters such that $\theta^0 = \theta^* + r u$, where $r = \max \{1, \enorm{\theta^*} \} \cdot 0.1$ and $u$ is a random unit vector. For measuring the statistical rate, the EM algorithm runs with different size of samples $n \in \{128, 180, 256, ...\}$ (approximately $\sqrt{2}$ times increased) and the final error is averaged over $5,000$ independent runs. The stopping criterion is the change in estimators being less than $0.0001$ in $l_2$ norm. In Figure \ref{figure:em_behavior} (a), we observe the standard $n^{-1/2}$ rate in the high SNR regime, and $n^{-1/4}$ rate in the low SNR regime. Interestingly, we can see a clear transition in the statistical rate when SNR = 0.3 as $n$ increases. This explains how the low SNR regime is defined $\enorm{\theta^*} \lesssim (d/n)^{1/4}$: the meaning of low SNR depends on how many samples we have, not on the absolute value that can be computed from a problem instance. 

We also look at the optimization complexity in Figure \ref{figure:em_behavior} (b, c). We run the EM algorithm with fixed sample size $n = 32768$. Estimation error $\enorm{\theta_n^t - \theta^*}$ in all iteration steps are averaged over $5,000$ independent runs. In the high SNR regime, note that the $y$-axis is in log-scale and we can see the linear convergence. In contrast, in the middle-to-low SNR regimes, we can observe that the convergence of the EM algorithm is no longer linear, and significantly slowed down. 
\begin{figure}[t]
    \centering
    \begin{tabular}{ccc}
        \includegraphics[width=48mm]{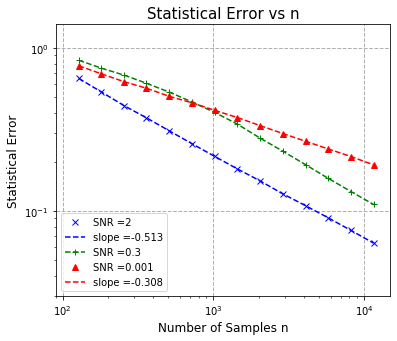} 
        \label{fig:rate_em} &
        \includegraphics[width=48mm]{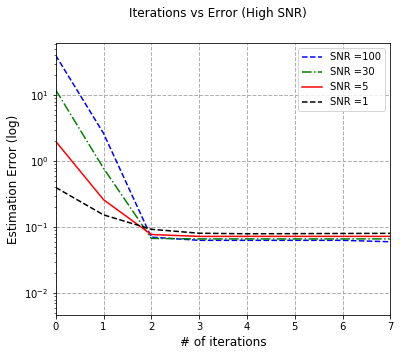}&
        \includegraphics[width=48mm]{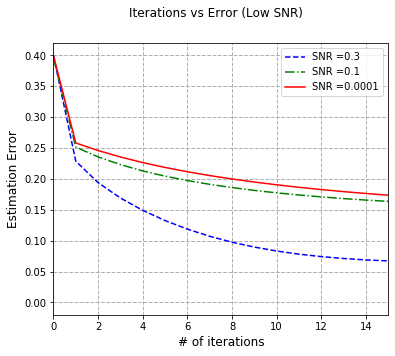} \\
        (a) & (b) & (c)
    \end{tabular}
    \caption{Convergence behavior of the EM algorithm for the fitted model \eqref{eq:mixed_regression_known_noise} when $d = 5$: (a) statistical rate ($\enorm{\theta_n^t - \thetastar}$ at the last iteration) in various SNRs (b) linear convergence in high SNR regime (c) slow convergence in low SNR regime.} 
    \label{figure:em_behavior}
\end{figure}

\subsection{Main results}
\label{subsec:main_results}
In this section, we state our main results with the convergence behaviors of the EM algorithm under different regimes of SNR. Our first result assumes a good initialization and focuses on the statistical optimality of the EM algorithm in the last iterations. We can use the standard spectral method to get such a good initialization (see Appendix \ref{Appendix:2MLR_init_spectral} for guarantees given by the spectral initialization). Then, with a mild condition on SNR and permission to use a simple variant of EM, our second result shows that EM converges globally to the true parameter with the same optimal statistical rates.

Throughout the paper, we assume that $n \ge Cd$ for sufficiently large constant $C>0$. Our analysis is divided into two cases when we are in the middle-high SNR regimes and low SNR regime. We state our first main theorem:
\begin{theorem}
    \label{theorem:EM_mlr_upper_bound}
    (a) (Middle-High SNR regimes) Suppose $\enorm{\theta^*} \ge C_0 (d \log^2(n/\delta) /n)^{1/4}$ for some large universal constant $C_0 > 0$. In this regime, suppose we run the EM algorithm starting from well-initialized $\theta_n^0$ such that $\enorm{\theta_n^0} \ge 0.9 \enorm{\theta^*}$ and $\cos \angle(\theta^*, \theta_n^0) \ge 0.95$. Then, for any $\delta > 0$ there exist universal constants $C_1, C_2 > 0$ such that the EM updates~\eqref{eq:EM_updates_mixed_regression_known_noise} give $\theta_n^t$ for $\theta^*$ which satisfies
    \begin{align*}
        \enorm{\theta_n^t - \theta^*} \le C_1 \max \{1, \enorm{\theta^*}^{-1} \} (d \log^2(n \enorm{\theta^*}/\delta) /n)^{1/2}, 
    \end{align*}
    with probability at least $1 - \delta$ after $t \ge C_2 \max \{1, \enorm{\theta^*}^{-2}\} \log(n\enorm{\theta^*}/d)$ iterations.
    
    \noindent
    (b) (Low SNR regime) When $\enorm{\theta^*} \le C_0 (d \log^2(n/\delta) /n)^{1/4}$, there exist universal constants $C_3, C_4 > 0$ such that the EM updates~\eqref{eq:EM_updates_mixed_regression_known_noise} initialized with $\enorm{\theta_n^0} \le 0.2$ return $\theta_n^t$ which satisfies
    \begin{align*}
        \enorm{\theta_n^t - \theta^*} \le C_3 (d \log^2 (n/\delta) /n)^{1/4}, 
    \end{align*}
    with probability at least $1 - \delta$ after $t \ge C_4 \log(\log(n/d)) \sqrt{n/(d \log^2 (n/\delta))} $ iterations.
\end{theorem}

The proof sketch of Theorem~\ref{theorem:EM_mlr_upper_bound} is in Section~\ref{sec:proof} while the full proof is in Appendix~\ref{subsec:proof_main_theorem}. Interestingly, the upper bound given by Theorem \ref{theorem:EM_mlr_upper_bound} matches the known lower bounds given for all SNR regimes in~\cite{chen2014convex}, and explains detailed behavior that interpolates between different separation regimes. Note that, the additional requirement $\enorm{\theta_n^0} \ge 0.9 \enorm{\theta^*}$ under middle-high SNR regimes is to prevent the analysis to become over-complicated (see Appendix~ \ref{subsec:proof:global_stability_convergence} for the arguments for starting from well-aligned small estimators). Furthermore, the initialization condition $\enorm{\theta_{n}^{0}} \leq 0.2$ in the low SNR regime is not restrictive. In Appendix~\ref{subsec:decrease_norm}, we demonstrate that when we initialize with large norm such that $\enorm{\theta_n^0} \ge 0.2$, in a finite number of steps the norm of EM updates becomes smaller than 0.2.

Next, we present our second result that does not rely on the warm start, but requires slightly more involved mechanisms. We call the following variant of EM as ``Easy-EM" operator \cite{kwon2019global}:
\begin{align}
    \label{eq:easy_em}
    M_{easy} (\theta) := \frac{1}{n} \sum_{i = 1}^{n} X_i Y_i \tanh(Y_i X_i^\top \theta).
\end{align}
Note that the only difference is the absence of the inverse of the sample covariance matrix. Our second theorem guarantees the global convergence of the EM algorithm with minimax optimality:
\begin{theorem}
    \label{theorem:EM_mlr_global}
    Given $C > 0$, suppose that $\enorm{\theta^*} \le C$. Let $\theta_n^0$ be a randomly initialized vector in $\mathbb{R}^d$ space such that the direction of $\theta_n^0$ is randomly sampled from a uniform distribution on the unit sphere. The norm of initial estimator can be any non-zero constant such that $\enorm{\theta_n^0} \ge c (d \log^2(n/\delta) / n)^{1/4}$ for some universal constant $c > 0$. 
    
    (a) In the middle-to-high SNR regimes, there exist universal constants $C_1, C_2, C_3 > 0$ such that when $C_1 (d \log^2(n/\delta) / n)^{1/4} \le \enorm{\theta^*} \le C$, with probability at least $1 - \delta$, we have 
    \begin{align*}
        \enorm{\theta_n^t - \theta^*} \le C_2 \max \{1, \enorm{\theta^*}^{-1} \} (d \log^2(n/\delta)/n)^{1/2},
    \end{align*}
    after we first run the Easy-EM algorithm~\eqref{eq:easy_em} for $C_3 \max\{ 1, \enorm{\theta^*}^{-2} \} \log (d)$ iterations, and then run the standard EM algorithm~\eqref{eq:sample_operator_mlr} for $C_3 \max\{ 1, \enorm{\theta^*}^{-2} \} \log (n/d)$ iterations.
    
    (b) In the low SNR regime when $\enorm{\theta^*} \le C_1 (d \log^2(n/\delta) / n)^{1/4}$, there exist universal constants $C_4, C_5 > 0$ such that with probability at least $1 - \delta$, we have
    \begin{align*}
        \enorm{\theta_n^t - \theta^*} \le C_4 (d \log^2(n/\delta)/ n)^{1/4}, 
    \end{align*}
    after we run either Easy-EM or standard EM for $t \ge C_5 \log (\log (n/d)) \sqrt{n/(d \log^2(n/\delta))}$ iterations.
\end{theorem}
The proof sketch of Theorem~\ref{theorem:EM_mlr_global} is in Section~\ref{sec:proof} while the full proof is in Appendix~\ref{Appendix:2MLR_init_global}. A few comments are in order. First, comparing to Theorem~\ref{theorem:EM_mlr_upper_bound}, we have an additional assumption for $\enorm{\theta^*}$ being bounded. This is required for a technical reason that arises from giving an uniform control on the deviation of Easy-EM operator in one direction when $\enorm{\theta^*}$ can be arbitrarily large (see Remark \ref{Remark:global_angle_boundedness} in Appendix \ref{subsec:proof:global_angle} for details). Second, in order to correctly estimate how many iterations we must run Easy-EM, we can check the value of $\frac{1}{n} \sum_{i = 1}^{n} Y_i^2 - 1$, since the expectation of this value is $\enorm{\theta^*}^2$. 
We note that Easy-EM is only introduced for a theoretical justification, and in practice we can just run the EM algorithm from a randomly initialized point. Finally, our condition on the norm of initial estimator is to ensure that the initial point is sufficiently far from zero. In practice, we use any constant $\Omega(1)$ for the norm of initial estimator. This is in stark contrast to the initialization of \cite{wu2019randomly} in which only very small initialization of order $\Theta((d/n)^{1/4})$ is allowed, which goes to 0 as $n \rightarrow \infty$.

\subsection{Tightness of the results}
\label{subsection:tightness}
In this section, we discuss in detail the tightness of our results in Theorem~\ref{theorem:EM_mlr_upper_bound} and Theorem~\ref{theorem:EM_mlr_global}.
\paragraph{Tightness of the result in the high SNR regime:}
In the high SNR regime, a minimax rate should guarantee exact recovery when the noise variance goes to zero. Our results obtain a statistical rate of $\sqrt{d \log^2 (n \enorm{\theta^*}/\delta)/ n}$. Note that, since we have rescaled to $\sigma^* = 1$, we should interpret the statistical rate of EM algorithm in the original scale where it is translated to $(\sigma^* \log(1/\sigma^*)) \sqrt{d \log^2 (n\enorm{\theta^*}/\delta)/ n}$. Therefore, we still guarantee the exact recovery as $\sigma^* \rightarrow 0$. We conjecture that a more careful and thorough analysis can also resolve even the logarithmic dependency on $\enorm{\theta^*}$, and leave it as future work. As mentioned earlier, there has been much recent interest in establishing the linear convergence and tight finite-sample error in high SNR regime \cite{yi2014alternating, yi2016solving, kwon2019global, kwon2020converges}. While all previous results are also minimax optimal in all parameters (up to logarithmic factors), as an artifact of their analysis, their results rely on sample-splitting, and thus do not in fact analyze the algorithm that is used in practice. Our results remove this artifact.


A very recent work in~\cite{ghosh2020alternating} has established a super-linear convergence of the EM algorithm in the noiseless setting (a.k.a.\ Alternating Minimization). In fact, we conjecture that their result can be extended to the noisy setting when SNR is high enough ({\it i.e.,} $\enorm{\theta^*} \gg 1$). The following lemma on the population EM operator~\eqref{eq:population_operator_mlr} gives a hint for a super-linear convergence in the high SNR regime:
\begin{lemma}
\label{lemma:super_linear_EM}
    If $C \sqrt{\log \enorm{\theta^*}} \le \enorm{\theta - \theta^*} \le \enorm{\theta^*} / 10$ for sufficiently large constant $C > 0$, then there exists a constant $c < 10$ such that
    \begin{align*}
        \enorm{\popopmlr(\theta) - \theta^*} \le c \enorm{\theta - \theta^*}^2/\enorm{\theta^*}.
    \end{align*}
\end{lemma}
The proof of Lemma~\ref{lemma:super_linear_EM} is in Appendix~\ref{subsec:super_linear_EM}. This lemma implies that until $\enorm{\theta - \theta^*}$ drops from $O(\enorm{\theta^*})$ to $O(\sqrt{\log \enorm{\theta^*}})$, the population EM updates converge in a super-linear rate. While the super-linear convergence behavior is a very interesting phenomenon and deserves further exploration, it is beyond the scope of this paper. 
\paragraph{Tightness of the result in the middle-low SNR regimes:}
As discussed in the introduction, \cite{kwon2019global} has recently established a convergence of the EM algorithm in SNR regimes for model \eqref{eq:em_independence_fit_known_noise}. In particular, according to the result in \cite{kwon2019global}, the EM algorithm can achieve arbitrary $\epsilon$ accuracy if the sample size $n$ is large enough to compensate a low SNR $\eta := \enorm{\theta^*}/ \sigma^{*}$, {\it i.e.,} $\eta^{-6} / \epsilon^2 \lesssim n$. This sub-optimal result is an artifact of the technical approach used to relate the population and finite-sample EM operators. Specifically, the convergence rate of the population EM operator is given by $1 - \eta^2$. The finite-sample analysis then follows by analyzing the uniform deviation of finite-sample operators from population operators, which is in order of magnitude $\sqrt{d/n}$. In order to guarantee the progress toward $\theta^*$ in each step as well as to control the accumulation of statistical errors in all iterations, \cite{kwon2019global} required $n \gtrsim \eta^{-6}$ {\it per} iteration. The sample-splitting results in even worse total $n \gtrsim \eta^{-8}$ sample complexity in terms of SNR. Furthermore, nothing can be explained when the sample size is less than the threshold $\eta^{-8}$. This calls for a more refined and tighter analysis of the EM algorithm in middle-to-low SNR regimes. 

We adopt the localization argument used in \cite{Raaz_Ho_Koulik_2018, Raaz_Ho_Koulik_2018_second} where they establish the convergence behaviors of the EM algorithm under over-specified Gaussian mixtures, namely, no separation of the parameters. Unlike these previous studies, our analysis is not restricted to strictly over-specified instances, but spans all possible configuration of parameters. The core of the analysis has three parts: (i) refined convergence rate of the population EM operator $1 - \max \{ \enorm{\theta}^2 - \eta^2, \eta^2\}$, (ii) multi-level application of uniform deviation of finite-sample operators that is proportional to $\enorm{\theta} \sqrt{d/n}$, and (iii) localization arguments applied to different levels of $\enorm{\theta}$. The threshold that separates middle-SNR and low-SNR regimes is naturally found at $\eta^2 = \sqrt{d/n}$.


\paragraph{Global Convergence of (Easy) EM:} 
Global convergence of the EM algorithm for model~\eqref{eq:mixed_regression_known_noise} has been established in \cite{kwon2019global} using the idea of two-phase analysis where EM first converges in angle, and then converges in $l_2$ norm. In the initial stage of the EM iterations with a random initialization, \cite{kwon2019global} proposed a simple variant of the EM update~\eqref{eq:easy_em} to encourage the boosting of angle from $\cos \angle(\theta_n^0, \theta^*) = O(1/\sqrt{d})$. 
Our result removes the usage of sample-splitting in \cite{kwon2019global} and tightens the sub-optimal statistical rate in middle-to-low SNR regimes as in Theorem \ref{theorem:EM_mlr_upper_bound}.

In \cite{wu2019randomly}, the authors employed a similar idea of analyzing the growth of the signal strength in the $\theta^*$ direction for learning a two symmetric mixture of Gaussian distributions. However, in general the value itself in $\theta^*$ direction can indeed decrease if EM starts from large initialization. Therefore, they restricted the initialization to be within a very small radius of $\enorm{\theta_n^0} \approx (d/n)^{1/4}$ in all SNR (separation) regimes. While it does not degrade the overall computational complexity of the finite-sample EM algorithm, the convergence guarantee with such small initialization is not global in a true sense since if $n$ grows to infinity ({\it i.e.,} approach to the population setting), the initialization should be at 0, which is a saddle point of the log-likelihood. Theorem \ref{theorem:EM_mlr_global} resolves the open issue of small initialization in \cite{wu2019randomly} by analyzing the convergence in angle.


%% file: Arxiv_proof_sketch.tex
\section{Overview of Techniques in Main Theorems}
\label{sec:proof}
\subsection{Proof Sketch of Theorem \ref{theorem:EM_mlr_upper_bound}}
In this section, we give a proof sketch of Theorem \ref{theorem:EM_mlr_upper_bound}. The full proof of Theorem~\ref{theorem:EM_mlr_upper_bound} is in Appendix~\ref{subsec:proof_main_theorem}. We need the following uniform deviation bound between sample and population EM operators:
    \begin{lemma} \label{lemma:deviation_bound_mlr}
        Given the population and finite-sample EM operators $\popopmlr$, $\samopmlr$ in equations~\eqref{eq:population_operator_mlr} and~\eqref{eq:sample_operator_mlr}, for any given $r > 0$, there exists a universal constant $c> 0$ such that we have
        \begin{align}
            \label{eq:uniform_deviation_mlr}
            \Prob \parenth{\sup \limits_{\enorm{\theta} \le r} \enorm{\samopmlr(\theta) - \popopmlr(\theta)} \le c r \sqrt{d \log^2 (n/\delta)/n}} \ge 1 - \delta.
        \end{align}
    \end{lemma}
    While the lemma is a straight-forward consequence of Lemma \ref{lemma:uniform_concentration_mlr} given in Appendix \ref{sec:concentration_sample_EM}, this is the first key result to get a tight statistical rate. The proof of Lemma \ref{lemma:deviation_bound_mlr} can be found in Appendix \ref{subsec:proof:lemma:deviation_bound_mlr}. 
    
    \paragraph{High SNR regime: \texorpdfstring{$\enorm{\theta^*} \ge C$}{Lg}.} 
    The high-level proof in the high SNR regime follows a specialized proof strategy exploited in \cite{kwon2020converges}. The core idea is that for high SNR, most ``good" samples are assigned correct (soft but almost hard) labels in E-step, and the portion of ``bad" samples is negligibly small. Such an argument first appeared informally in \cite{Siva_2017}, and then was formally organized in \cite{kwon2020converges, kwon2020algorithm} to establish a linear convergence and tight statistical rate. The full proof for the high SNR regime is given in Appendix~\ref{subsec:full_proof_high_SNR}.

    \paragraph{Middle SNR regime: \texorpdfstring{$C_{0}(d \log^2(n/\delta) /n)^{1/4} \leq \enorm{\theta^*} \le C$}{Lg}.} 
    We consider two cases, when $\enorm{\theta^*} \ge 1$ and $\enorm{\theta^*} \le 1$. 
    
    {\bf Case (i) $1 \le \enorm{\theta^*} \le C$:} Given the initialization conditions in Theorem \ref{theorem:EM_mlr_upper_bound}, we can show that $\enorm{\popopmlr (\theta) - \theta^*} < 0.9 \enorm{\theta - \theta^*}$. Furthermore, from the uniform concentration Lemma \ref{lemma:uniform_concentration_mlr} in Appendix~\ref{sec:concentration_sample_EM}, we have $\enorm{M_{n, mlr}(\theta) - M_{mlr}(\theta)} \leq \sqrt{d \log^2 (n/\delta)/n}$ with probability at least $1 - \delta$. From here, we can check that
    \begin{align*}
        \enorm{\theta_n^t - \theta^*} \lesssim \parenth{0.9}^t \enorm{\theta - \theta^*} + {\sqrt{d \log^2(n/\delta) /n}}.
    \end{align*}
    
    {\bf Case (ii) $C_{0}(d \log^2(n/\delta) /n)^{1/4} \leq \enorm{\theta^*} \le 1$:} In this case, the result of Lemma~\ref{lemma:theorem4_global} in Appendix~\ref{subsec:proof_main_theorem} shows that
    \begin{align}
        \label{eq:pop_rate_mlr_middle_snr}
        \enorm{M_{mlr}(\theta) - \theta^*} &\le \parenth{1 - O( \enorm{\theta^*}^2)} \enorm{\theta - \theta^*}.
    \end{align}
    As Lemma \ref{lemma:deviation_bound_mlr} and Corollary \ref{corollary:middle_SNR_contraction} in the Appendix make precise, we can infer that in order for the EM algorithm to make progress toward $\theta^*$, we need $\enorm{\theta^*}^2 \enorm{\theta - \theta^*} \gtrsim \enorm{\theta} \sqrt{d/n}$. Intuitively, EM converges to $\theta^*$ as long as such a relation holds, and until $\theta$ gets close enough to $\theta^*$ such that the above equation does not hold. In other words, in the last iterations when $\enorm{\theta} \approx \enorm{\theta^*}$, we have
    \begin{align*}
        \enorm{\theta^*}^2 \enorm{\theta - \theta^*} \approx \enorm{\theta^*} \sqrt{d/n}, 
    \end{align*}
    which implies the statistical rate should be on the order of $\enorm{\theta^*}^{-1} \sqrt{d/n}$. The full proof is given in Appendix~\ref{subsec:full_proof_middle_SNR}.

\paragraph{Low SNR Regime: \texorpdfstring{$\enorm{\theta^*} \leq C_{0} (d \log^2(n/\delta) /n)^{1/4}$}{Lg}.} 
In this case, even the standard spectral methods would not give a good initialization since the eigenspace is perturbed too much to be aligned with $\theta^*$ (see Lemma~\ref{lemma:spectral_init} in Appendix~\ref{Appendix:2MLR_init_spectral} for the guarantees given by spectral methods). Instead, we assume the initial estimator to be $\enorm{\theta_n^0} \le 0.2$. 

The core of idea of the low SNR regime is that EM essentially cannot distinguish the cases between $\theta^* = 0$ and $\theta^* \neq 0$. Therefore, we aim to investigate $\enorm{\theta}$ instead of the estimation error $\enorm{\theta - \theta^*}$. If we can show that $\enorm{\theta_n^t} \leq c_{1} \cdot (d/n)^{1/4}$, then given the condition of low SNR regime, we have $\enorm{\theta_n^t - \theta^*} \leq c_{2} \cdot (d/n)^{1/4}$ where $c_{1}, c_{2}$ are some positive constants.

In the low SNR regime, there exist universal constants $c_l, c_u > 0$ such that for $\enorm{\theta} \le 0.2$, we have
\begin{align*}
    \enorm{\theta} (1 - 4 \enorm{\theta}^2 - c_l \enorm{\theta^*}^2) \le \enorm{\popopmlr(\theta)} \le \enorm{\theta} (1 - \enorm{\theta}^2 + c_u \enorm{\theta^*}^2).
\end{align*}
The statistical fluctuation of the finite-sample EM operator given in Lemma \ref{lemma:deviation_bound_mlr} shows that $\enorm{\samopmlr(\theta ) - \popopmlr(\theta)} \leq c \cdot \enorm{\theta} \sqrt{d \log^2(n/\delta)/n}$,
for some universal constant $c$. It is now more clear to see that since $\enorm{\theta^*}^2 \lesssim \sqrt{d/n}$, the above statistical error will subsume an extra $O(\enorm{\theta^*}^2)$ term in the contraction rate of the population EM operator. Therefore, the convergence behaviors of the finite-sample EM operator are essentially the same when $\theta^* = 0$ and $\theta^* \neq 0$.

The EM iterations stop improving the estimator when the statistical error becomes larger than the amount that the population EM can proceed:
\begin{align*}
    \enorm{\theta}^2 \approx \sqrt{d \log^2(n/\delta) /n}.
\end{align*}
Therefore, the statistical rate of the EM algorithm is achieved at $\enorm{\theta} \lesssim (d/n)^{1/4}$. The rest of the proof in the low SNR regime is a reminiscent of the localization arguments used in \cite{Raaz-misspecified,Raaz_Ho_Koulik_2018_second}, and can be found in Appendix \ref{subsec:full_proof_low_SNR}.

\subsection{Proof Sketch of Theorem \ref{theorem:EM_mlr_global}}
The global convergence statement is subsumed into Theorem \ref{theorem:EM_mlr_upper_bound} when the estimator $\theta$ enters in the initialization region that Theorem \ref{theorem:EM_mlr_upper_bound} requires. Therefore we can focus on the iterations that $\theta$ stays outside of the initialization region. The key idea is to adopt the angle convergence argument presented in \cite{kwon2019global}. Note that in low SNR regime, we do not need such an involved argument since the initialization only requires $\enorm{\theta_n^0} \le 0.2$ (see Appendix~\ref{subsec:decrease_norm} for an argument why this initialization is easy to satisfied). In middle SNR regime where $(d/n)^{1/4} \lesssim \enorm{\theta^*} \le 1$, the key property is that
$$\cos \angle(\popopmlr(\theta), \theta^*) \ge (1 + c \enorm{\theta^*}^2) \cos \angle(\theta, \theta^*),$$
for some universal constant $c > 0$. We again see that the increase rate is $1 + O(\enorm{\theta^*}^2)$; however, the cosine value is very small $\Theta(1/\sqrt{d})$ at the initial stage. Then, the second key step is to show that $$\cos \angle (M_{easy}(\theta) - M_{mlr} (\theta), \theta^*) \le \epsilon_f / \sqrt{d},$$ for sufficiently small $\epsilon_f \lesssim \sqrt{d/n}$. At a high level, if it holds that $c\enorm{\theta^*}^2 \cos \angle(\theta, \theta^*) \ge 2 \epsilon_f / \sqrt{d}$, then we can guarantee that $\cos \angle (M_{easy}(\theta), \theta^*) \ge (1 + c\enorm{\theta^*}^2 / 2) \cos \angle (\theta, \theta^*)$. We can conclude that this is true in the middle-SNR regime since $\enorm{\theta^*}^2 \gtrsim (d/n)^{1/2}$. The argument in high-SNR regime is similar to middle-SNR regime. The formal proof is a bit more involved since we need to ensure that the statistical error in orthogonal directions does not dominate the angle (see Appendix \ref{subsec:proof:global_angle} for more detail).

%% file: Arxiv_appendix.tex
\section{Additional Notations}
\label{Appendix:additional_notation}
We sometimes use the transformed coordinate where the first two coordinate spans $\theta$ and $\theta^*$. That is, let $\{v_1, ..., v_d\}$ be standard basis in the transformed coordinate such that $v_1 = \theta / \enorm{\theta}$, and $\textbf{span}(v_1, v_2) = \textbf{span}(\theta, \theta^*)$. Since Gaussian distribution is invariant to rotation, we often work on the transformed space in the proofs. Let $\alpha = \angle (\theta, \theta^*)$, $\eta = \enorm{\theta^*}/\sigma^*$, and $\sigma_2^2 = 1 + \enorm{\theta^*}^2 \sin^2 \alpha$.

We define a few more quantities to simplify the notations throughout the proofs. Let $x_1, x_2$ be $X^\top v_1, X^\top v_2$ respectively. Following the notation in \cite{kwon2019global}, we denote $b_1^* = {\theta^*}^\top v_1 = \enorm{\theta^*} \cos \angle(\theta, \theta^*)$, and $b_2^* = {\theta^*}^\top v_2 = \enorm{\theta^*} \sin \angle(\theta, \theta^*)$. Note that in this transformed coordinate, due to the symmetry of the distribution, $M_{mlr}(\theta)^\top v_j = 0$ for all $j \ge 3$. Hence we focus on bounding the values in first two coordinates.

Using the coordinate transformation and new notations defined here, we can write the population operator in new coordinate as:
\begin{align}
    \label{eq:pop_mlr_transformed}
    M_{mlr}(\theta) &= \Exs_{X, Y} \brackets{\tanh(Y X^\top \theta) Y X} \nonumber \\
    &= \Exs_{x_1, x_2, y} \brackets{\tanh(y x_1 \enorm{\theta}) x_1 y} v_1 + \Exs_{x_1, x_2, y} \brackets{\tanh(y x_1 \enorm{\theta}) x_2 y} v_2,
\end{align}
where $y|(x_1, x_2) \sim \NORMAL(x_1 b_1^* + x_2 b_2^*, 1)$. Note that we simplify $y$ as a single Gaussian due to the symmetry in the signs of $y$ and Gaussian noise.

\section{Proof of Theorem~\ref{theorem:EM_mlr_upper_bound}}
\label{subsec:proof_main_theorem}
We first consider middle-to-high SNR regimes and then we consider low SNR regimes. In middle-to-high SNR regimes, we assume that we start from the initialization where $\cos \alpha \ge 0.95$. We note that the additional requirement $\enorm{\theta_n^0} \ge 0.9 \enorm{\theta^*}$ is to prevent the analysis to become over-complicated (see Appendix \ref{subsec:proof:global_stability_convergence} for the arguments for starting from well-aligned small estimators).

We will frequently use the fact that $\enorm{\theta^*} \sin \alpha \le \enorm{\theta - \theta^*}$. We can check that $\theta$ remains in this good initialization region using the convergence property of angles (see the arguments for sine values in Appendix \ref{subsec:proof:global_stability_convergence}). Before getting into the detailed proof, we state some useful lemmas from previous work. We need the following lemma for the contraction rate of the population EM operator~\eqref{eq:population_operator_mlr}:
\begin{lemma}[Theorem 4 in \cite{kwon2019global}]
    \label{lemma:theorem4_global}
    Assume $\alpha < \pi / 8$. Then, we have
    \begin{align}
        \label{eq:population_contract_mlr}
        \enorm{M_{mlr}(\theta) - \theta^*} \le \max \{\kappa, 0.6\} \enorm{\theta - \theta^*} + \kappa (16 \sin^3 {\alpha}) \enorm{\theta^*} \frac{ \eta^2}{1 + \eta^2},
    \end{align}
    where $\kappa = \parenth{\sqrt{1 + \min \{\sigma_2^2 \enorm{\theta}, \enorm{\theta^*} \cos\alpha \}^2 / \sigma_2^2}}^{-1}$.
\end{lemma}

\subsection{High SNR Regime}
\label{subsec:full_proof_high_SNR}
First, we arrange the sample operator as the following:
\begin{align}
    &M_{n, mlr} (\theta) - \theta^* = \parenth{\frac{1}{n} \sum_i X_i X_i^\top}^{-1} \parenth{\frac{1}{n} \sum_i X_i Y_i \tanh(Y_i X_i^\top {\theta})} - \theta^* \nonumber \\
    &= \parenth{\frac{1}{n} \sum_i X_i X_i^\top}^{-1} \biggr( \frac{1}{n} \sum_i X_i Y_i \tanh(Y_i X_i^\top {\theta}) - \frac{1}{n} \sum_i X_i Y_i \tanh(Y_i X_i^\top {\theta^*}) \nonumber \\
    &\qquad \qquad \qquad \qquad + \frac{1}{n} \sum_i X_i Y_i \tanh(Y_i X_i^\top {\theta^*}) - \frac{1}{n} \sum_i X_i X_i^\top \theta^* \biggr) \nonumber \\
    &= \parenth{\frac{1}{n} \sum_i X_i X_i^\top}^{-1} \Biggr( \underbrace{\Exs_{X, Y} [XY \Delta_{(X,Y)}(\theta) ]}_{:= A_{1}} + \underbrace{\frac{1}{n} \sum_i X_i Y_i \Delta_{(X_i, Y_i)}(\theta) - \Exs_{X, Y} [XY \Delta_{(X, Y)} (\theta)]}_{:=A_{2}}  \nonumber \\
    &\qquad \qquad \qquad + \underbrace{\frac{1}{n} \sum_i X_i Y_i \tanh(Y_i X_i^\top {\theta^*}) - \Exs_{Y_i | X_i} \brackets{\frac{1}{n} \sum_i X_i Y_i \tanh(Y_i X_i^\top {\theta^*})}}_{:= A_{3}} \Biggr), \label{eq:key_high_SNR}
\end{align}
where $\Delta_{(X,Y)}(\theta) : = \tanh(YX^\top \theta) - \tanh(YX^\top \theta^*)$. In the term $A_{3}$, the expectation is taken over $Y_i|X_i \sim \frac{1}{2} \NORMAL(X_i^\top \theta^*, 1) + \frac{1}{2} \NORMAL(-X_i^\top \theta^*, 1)$, letting $X_i$ fixed. Note that the true parameters are fixed points of the EM operators, and it is easy to check that the expectation in $A_{3}$ is equivalent to $\frac{1}{n} \sum_i X_i X_i^\top \theta^*$.

Now, we claim the following bounds with $A_{1}, A_{2}$, and $A_{3}$ in equation~\eqref{eq:key_high_SNR}:
\begin{align}
    A_{1} & < 0.9 \enorm{\theta - \theta^*}, \label{eq:claim_one_high_SNR} \\
    A_{2} & \leq (\enorm{\theta - \theta^*} + 1)\sqrt{d\log^2 (n \enorm{\theta^*} /\delta) /n}, \label{eq:claim_two_high_SNR} \\
    A_{3} & \leq C \sqrt{d \log(1/ \delta)/n}, \label{eq:claim_three_high_SNR}
\end{align}
with probability at least $1 - 5 \delta$. Here, $C$ is some universal constant.

Assume that the above claims are given at the moment, we proceed to finish the proof of the convergence of EM algorithm under high SNR regime. In fact, plugging the results from equations~\eqref{eq:claim_one_high_SNR},~\eqref{eq:claim_two_high_SNR}, and~\eqref{eq:claim_three_high_SNR} into equation~\eqref{eq:key_high_SNR}, we find that
\begin{align*}
    \enorm{M_{n,mlr}(\theta) - \theta^*} &\leq \parenth{0.9 + \sqrt{d \log^2 (n \enorm{\theta^*} /\delta) /n})} \enorm{\theta - \theta^*} + C_{1} \sqrt{d \log^2 (n \enorm{\theta^*} /\delta) /n} \\
    &\le \gamma \enorm{\theta - \theta^*} + C_{1} \sqrt{d \log^2 (n \enorm{\theta^*} /\delta) /n},
\end{align*}
for some $\gamma < 1$. From here, let $\epsilon_n := C_{1} \sqrt{d \log^2 (n \enorm{\theta^*} /\delta) /n}$ and we iterate over $t$ to bound the estimation error in $t^{th}$ step:
\begin{align*}
    \enorm{\theta_n^{t+1} - \theta^*} &\le \gamma \enorm{\theta_n^t - \theta^*} + \epsilon_n \le \gamma^2 \enorm{\theta_n^{t-1} - \theta^*} + (1 + \gamma) \epsilon_n \\
    &\le ... \le \gamma^t \enorm{\theta_n^0 - \theta^*} + \frac{1}{1-\gamma} \epsilon_n.
\end{align*}
After $t \geq c_{1} \log (n \enorm{\theta^*} /d)$ iterations, we have $\enorm{\theta_n^t - \theta^*} \leq c_{2} \sqrt{d/n}$ where $c_{1}$ and $c_{2}$ are universal constants. As a consequence, we reach the conclusion of the theorem for high SNR regime.

\paragraph{Proof of claim~\eqref{eq:claim_one_high_SNR}:} In order to bound $A_{1}$, we can use the result of Corollary~\ref{corollary:high_SNR_contraction} in Appendix~\ref{subsec:full_proof_middle_SNR}. Observe that
\begin{align*}
    \Exs \brackets{XY \tanh(YX^\top \theta^*)} = \theta^*, \\
    \Exs \brackets{XY \tanh(YX^\top \theta)} = M_{mlr} (\theta).
\end{align*}
From Corollary~\ref{corollary:high_SNR_contraction}, we conclude that 
\begin{align*}
    A_{1} = \Exs_{X,Y} [XY\Delta_{(X,Y)}(\theta)] < 0.9 \enorm{\theta - \theta^*}.
\end{align*}
Therefore, we reach the conclusion of claim~\eqref{eq:claim_one_high_SNR}.

\paragraph{Proof of claim~\eqref{eq:claim_two_high_SNR}:} Next, we bound $A_{2}$. We first discretize the parameter space for $\theta$ as the following:
\begin{align*}
    \Prob \biggr( \sup_{\theta \in \ball(\theta^*, r)} & \enorm{\frac{1}{n} \sum_{i=1}^n X_iY_i \Delta_i(\theta) - \Exs[XY\Delta(\theta)]} \ge t \biggr) \\
    &= \underbrace{\Prob \parenth{\sup_{j \in [\mathcal{N}_\epsilon]} \enorm{\frac{1}{n} \sum_{i=1}^n X_iY_i \Delta_i(\theta_j) - \Exs[X_iY_i \Delta_i(\theta_j)]} \ge t/2 }}_{\text{finite-sample error}} \\
    &+ \underbrace{\Prob \parenth{\sup_{\enorm{\theta-\theta'} \le \epsilon} \enorm{\frac{1}{n} \sum_{i=1}^n X_i Y_i (\Delta_i(\theta) - \Delta_i(\theta')} + \enorm{\Exs \brackets{XY(\Delta(\theta) - \Delta(\theta'))}} \ge t/2}}_{\text{discretization error}},
\end{align*}
where $\Delta_i(\theta)$ is a shorthand for $\Delta_i(\theta) := \tanh(Y_i X_i^\top \theta) - \tanh(Y_i X_i^\top \theta^*)$, $\Delta(\theta)$ is a shorthand for $\Delta(\theta) = \tanh(Y X^\top \theta) - \tanh(Y X^\top \theta^*)$, $\mathcal{N}_\epsilon$ is $\epsilon$-covering number of $\ball(\theta^*, r)$, and $\{\theta_j, j \in [\mathcal{N}_\epsilon]\}$ is the corresponding $\epsilon$-covering set. 

The discretization error can be bounded by the Lipschitz continuity of the function $\Delta_{i}$, namely, $|\Delta_i(\theta) - \Delta_i(\theta')| \le |Y_i| |X^\top \theta - X^\top \theta'|$ for all $\theta, \theta'$. It follows that 
\begin{align*}
    \enorm{\frac{1}{n} \sum_{i=1}^n X_i Y_i (\Delta_i(\theta) - \Delta_i(\theta')} &\le \enorm{\frac{1}{n} \sum_{i=1}^n Y_i^2 X_i X_i^\top (\theta - \theta')} \\
    &\le \epsilon \opnorm{\frac{1}{n} \sum_{i=1}^n Y_i^2 X_i X_i^\top}.
\end{align*}
Note that $\Exs[Y^2 XX^\top] = I + 2 \theta^* {\theta^*}^\top$, hence $\opnorm{\Exs[Y^2 XX^\top]} \le 2 \enorm{\theta^*}^2 + 1$. 
Furthermore, from Lemma~\ref{lemma:concentration_higher_order}, we have $\opnorm{\frac{1}{n} \sum_{i=1}^n Y_i^2 X_i X_i^\top} \le 3 \enorm{\theta^*}^2$ with probability at least $1 - \delta$. We conclude that 
\begin{align*}
    \text{discretization error} \leq 6\epsilon \enorm{\theta^*}^2
\end{align*}
with probability at least $1 - \delta$.

In order to bound the finite-sample error for each fixed $\theta_j$, we adopt the per-sample decomposition argument used in the previous works~\cite{kwon2020converges} and~\cite{kwon2020algorithm}. In order to simplify the notation, let $Z_i$ be the noise such that $Y_i = \nu_i X_i^\top \theta^* + Z_i$ where $\nu_i$ is an independent Rademacher variable. We define good events as follows:
\begin{align*}
    \event_{1} &= \{2 |X^\top (\theta^* - \theta)| \le |X^\top \theta^*|\}, \\
    \event_{2} &= \{|X^\top \theta^*| \ge 2\tau \}, \\
    \event_{3} &= \{|Z| \le \tau \}, 
\end{align*}
where we decide $\tau$ later. Let the good event $\event_{good} : = \event_{1} \cap \event_{2} \cap \event_{3}$. Then we have a following lemma:
\begin{lemma}
    \label{lemma:good_event_small_diff}
    Under the event $\event_{good}$, we have
    \begin{align*}
        |\Delta_{(X,Y)} (\theta)| \le \exp(-\tau^2).
    \end{align*}
\end{lemma}
\begin{proof}
    Without loss of generality, let $\nu = +1$. We can check that 
    \begin{align*}
        Y X^\top \theta &= (\nu X^\top \theta^* + Z) (X^\top \theta^*) + (\nu X^\top \theta^* + Z) (X^\top (\theta - \theta^*)) \\
        &= (\nu X^\top \theta^* + Z) (X^\top \theta^* + X^\top (\theta - \theta^*)) \\
        &\ge \tau \cdot \tau = \tau^2.
    \end{align*}
    Since $\tanh(x) = \frac{\exp(x) - \exp(-x)}{\exp(x) + \exp(-x)} \ge 1 - \exp(-x)$ for $x \ge 0$, we have $\tanh(YX^\top \theta) \ge 1 - \exp(-\tau^2)$. Similarly, $\tanh(YX^\top \theta^*) \ge 1 - \exp(-\tau^2)$. On the other hand, $\tanh(x) \le 1$ for all $x$. We can conclude that $\Delta_{(X,Y)} (\theta) \le \exp(-\tau^2)$. For the other sign $\nu = -1$, we can show it similarly.
\end{proof}
To simplify the notation, we denote $W_i : = \nu_i X_i X_i^\top \theta^* \Delta_i(\theta)$. Then, we can decompose $A_{2}$ as follows:
\begin{align}
    A_{2} = \underbrace{\parenth{\frac{1}{n} \sum_{i = 1}^{n} X_i Z_i \Delta_i(\theta) - \Exs[X Z \Delta(\theta)]}}_{:= T_{1}} + \underbrace{\parenth{\frac{1}{n} \sum_{i = 1}^{n} W_i - \Exs[W]}}_{:= T_{2}}. \label{eq:key_A2}
\end{align}
We first claim the following high probability bound with $T_{1}$:
\begin{align}
    \label{eq:high_SNR_concentration_noise_term}
    \Prob \parenth{\enorm{T_{1}} \ge t} \le \exp \parenth{-\frac{n t^2}{K_0} + K_0' d},
\end{align}
for some universal constants $K_0, K_0' > 0$, where we assumed $n \gg d$ to ignore sub-exponential tail part. The proof of claim~\eqref{eq:high_SNR_concentration_noise_term} is deferred to the end of the proof of high SNR regime.

For the term $T_{2}$ in equation~\eqref{eq:key_A2}, we apply per-sample decomposition.
\begin{align*}
    \frac{1}{n} \sum_i W_i - \Exs[W] &= \frac{1}{n} \sum_i (W_i 1_{\event_{good}} - \Exs[W 1_{\event_{good}}]) + \frac{1}{n} \sum_i (W_i 1_{\event_{1}^c} - \Exs[W 1_{\event_1^c}]) \\ 
    &+ \frac{1}{n} \sum_i (W_i 1_{\event_{1} \cap \event_2^c} - \Exs[W 1_{\event_{1} \cap \event_{2}^c}]) + \frac{1}{n} \sum_i (W_i 1_{\event_{1} \cap \event_2 \cap \event_3^c} - \Exs[W 1_{\event_{1} \cap \event_2 \cap \event_3^c}]).
\end{align*}
In the sequel, we will show that
\begin{align}
    &\Prob \parenth{\enorm{ \frac{1}{n} \sum_i (W_i 1_{\event_{good}} - \Exs[W 1_{\event_{good}}])} \ge t} \le \exp \parenth{- \frac{n t^2}{K_1 \enorm{\theta^*}^2 \exp(-2 \tau^2)} + K_1' d}, \label{eq:uniform_mlr_high_case1} \\
    &\Prob \parenth{\enorm{\frac{1}{n} \sum_i (W_i 1_{\event_{1}^c} - \Exs[W 1_{\event_1^c}])} \ge t} \le \exp \parenth{- \frac{nt^2}{K_2 \enorm{\theta - \theta^*}^2} + K_2' d}, \label{eq:uniform_mlr_high_case2} \\
    &\Prob \parenth{\enorm{\frac{1}{n} \sum_i (W_i 1_{\event_{1} \cap \event_2^c} - \Exs[W 1_{\event_{1} \cap \event_{2}^c}])} \ge t} \le \exp \parenth{-\frac{nt^2}{K_3 \tau^2} + K_3' d}, \label{eq:uniform_mlr_high_case3} \\
    &\Prob \biggr(\sup_{\theta \in \ball(\theta^*, r)} \enorm{\frac{1}{n} \sum_i (W_i 1_{\event_{1} \cap \event_2 \cap \event_3^c} - \Exs[W 1_{\event_{1} \cap \event_2 \cap \event_3^c}])} = 0 \biggr) \ge 1 - \delta, \label{eq:uniform_mlr_high_case4}
\end{align}
where $K_{(\cdot)}$ are all some universal constants. The last probability is due to our choice $\tau = \Theta(\sqrt{\log (n\|\theta^*\|/\delta)})$ such that no sample fall in the event $\event_3^c$ with probability at least $1 - \delta$. We set $t$ and $\epsilon$ as follows:
\begin{align*}
    t &= O\parenth{(\enorm{\theta - \theta^*} + 1) \sqrt{d \log^2 (n \enorm{\theta^*} /\delta) /n}}, \\
    \epsilon &= O \parenth{\enorm{\theta^*}^{-2} \sqrt{d \log^2 (n\|\theta^*\|/\delta) / n}}. 
\end{align*}
The overall finite-sample error term is bounded by taking union bound over $\epsilon$-covering set. Note that $\log (\mathcal{N}_\epsilon) \leq c \cdot d\log(\enorm{\theta^*})$ for some universal constant $c$. Hence the total probability of $\enorm{T_2} \ge t$ is dominated by
\begin{align*}
    \exp \parenth{-\frac{nt^2}{K_2 \enorm{\theta-\theta^*}^2} + K_2' d \log(n\enorm{\theta^*}/d)} + \exp \parenth{-\frac{nt^2}{K_3 \tau^2} + K_3' d \log(n\enorm{\theta^*}/d)},
\end{align*}
for some (new) constants $K_2, K_2', K_3, K_3' > 0$. Our choice of $t$ gives $5\delta$ total probability bound for the finite-sample error. We can conclude that $A_{2} \le t \le (\enorm{\theta - \theta^*} + 1)\sqrt{d\log^2 (n \enorm{\theta^*} /\delta) /n}$ with probability at least $1 - 5\delta$. Hence, we reach the conclusion of claim~\eqref{eq:claim_two_high_SNR}. 
\paragraph{Proof of claim~\eqref{eq:claim_three_high_SNR}:} Finally, for bounding $A_{3}$, we use Proposition 11 in~\cite{kwon2019global} that exactly targets to bound this quantity.
\begin{lemma}[Proposition 11 in \cite{kwon2019global}]
    For each fixed $\theta$, with probability at least $1 - \exp(-cn) - 6^d \exp(-nt^2/72)$, 
    \begin{align}
        \enorm{\frac{1}{n} \sum_i X_i Y_i \tanh(Y_i X_i^\top \theta) - \frac{1}{n} \sum_i \Exs_{Y_i|X_i} \brackets{Y_i X_i \tanh(Y_i X_i^\top \theta)}} \le t,
    \end{align}
    for some absolute constant $c > 0$.
\end{lemma}
Applying the above lemma for $\theta = \theta^*$, we can show that $A_{3} \leq C \sqrt{d \log(1/\delta)/n}$ with probability at least $1 - \delta$. As a consequence, we obtain claim~\eqref{eq:claim_three_high_SNR}.

\paragraph{Proof of Equation~\eqref{eq:high_SNR_concentration_noise_term}.}
We use the notion of sub-exponential Orcliz norm to bound \eqref{eq:high_SNR_concentration_noise_term}. It is easy to see that $X_i Z_i \Delta_i$ is a sub-exponential random vector with Orcliz norm $O(1)$. Using the standard concentration result in \cite{Vershynin_2011}, we get the result.

\paragraph{Proof of Equation~\eqref{eq:uniform_mlr_high_case1}.} Similarly to the previous case, we need to bound the sub-exponential norm of the quantity:
\begin{align*}
    \vecnorm{W_i 1_{\event_{good}}}{\psi_1} &= \sup_{u \in \sphere^{d-1}} \sup_{p \ge 1} p^{-1}
    \Exs \brackets{ |(X_i^\top u) (X_i^\top \theta^*) \Delta_i 1_{\event_{good}}|^{p} }^{1/p} \\
    &\le \exp(-\tau^2) \sup_{u \in \sphere^{d-1}} \sup_{p \ge 1} p^{-1}
    \Exs \brackets{ |(X_i^\top u) (X_i^\top \theta^*) |^{p} }^{1/p} \\
    &\le \exp(-\tau^2) \sup_{u \in \sphere^{d-1}} \sup_{p \ge 1} p^{-1}
    \sqrt{\Exs[(X^\top u)^{2p}] \Exs \brackets{(X_i^\top \theta^*)^{2p}} }^{1/p} \\
    &\le K_0 \enorm{\theta^*} \exp(-\tau^2).
\end{align*}
We use the fact that $|\Delta_i(\theta)| \le \exp(-\tau^2)$ under the good event, Cauchy-Schwartz inequality, and $p^{th}$-order moments of Gaussian is $O((2p)^{p/2})$. Similarly using the result in \cite{Vershynin_2011}, we have the equation \eqref{eq:uniform_mlr_high_case1}.

\paragraph{Proof of Equation~\eqref{eq:uniform_mlr_high_case2}.} We check the sub-exponential $\psi_1$-Orcliz norm again.
\begin{align*}
    \vecnorm{W_i 1_{\event_1^c}}{\psi_1} &= \sup_{u \in \sphere^{d-1}} \sup_{p \ge 1} p^{-1}
    \Exs \brackets{ |(X_i^\top u) (X_i^\top \theta^*) \Delta_i 1_{\event_1^c}|^{p} }^{1/p} \\
    &\le \sup_{u \in \sphere^{d-1}} \sup_{p \ge 1} p^{-1}
    \Exs \brackets{ |(X_i^\top u) (X_i^\top (\theta^* - \theta))|^{p} }^{1/p} \\
    &\le K_1 \enorm{\theta^* - \theta},
\end{align*}
from which we again use the standard result to get \eqref{eq:uniform_mlr_high_case2}.

\paragraph{Proof of Equation~\eqref{eq:uniform_mlr_high_case3}.} 
\begin{align*}
    \vecnorm{W_i 1_{\event_1 \cap \event_2^c}}{\psi_1} &= \sup_{u \in \sphere^{d-1}} \sup_{p \ge 1} p^{-1}
    \Exs \brackets{ |(X_i^\top u) (X_i^\top \theta^*) \Delta_i 1_{\event_1 \cap \event_2^c}|^{p} }^{1/p} \\
    &\le \sup_{u \in \sphere^{d-1}} \sup_{p \ge 1} p^{-1}
    \Exs \tau \brackets{ |(X_i^\top u)|^{p} }^{1/p} \\
    &\le K_2 \tau,
\end{align*}
getting the desired result.

\paragraph{Proof of Equation~\eqref{eq:uniform_mlr_high_case4}.} 
For this quantity, note that 
$$P(\forall i \in [n], |Z_i| \lesssim \log(n/\delta)) \ge 1 - n \exp(-\tau^2).$$
Hence it is very likely that no sample falls into this category. Meanwhile, we can bound the expectation term:
\begin{align*}
    \sup_{u \in \sphere^{d-1}} \Exs[W^\top u 1_{\event_1 \cap \event_2 \cap \event_3^c}] &\le \sup_{u \in \sphere^{d-1}} \Exs[(W^\top u) 1_{\event_1 \cap \event_2} | \event_3^c] P(\cap \event_3^c) \\
    &\le \sup_{u \in \sphere^{d-1}} \Exs[|(X_i^\top u)(X_i^\top \theta^*) 1_{\event_1 \cap \event_2}| | \event_3^c] P(\event_3^c) \\
    &\le \sup_{u \in \sphere^{d-1}} \Exs[|(X_i^\top u)(X_i^\top \theta^*)|] P(\event_3^c) \\
    &\le K_4 \enorm{\theta^*} \exp(-\tau^2).
\end{align*}
Since $\tau = \Theta(\log (n \enorm{\theta^*}/\delta))$, we have the result.
\subsection{Middle SNR Regime}
\label{subsec:full_proof_middle_SNR}
We consider two cases, when $\enorm{\theta^*} \ge 1$ and $\enorm{\theta^*} \le 1$. 
    
    {\bf Case (i) $1 \le \enorm{\theta^*} \le C$:} Given the initialization conditions in Theorem \ref{theorem:EM_mlr_upper_bound}, we can get the following corollary of Lemma \ref{lemma:theorem4_global}.
    \begin{corollary}
        \label{corollary:high_SNR_contraction}
        When $\enorm{\theta^*} \ge 1$ and $\sin \alpha < 0.1$, we have
        \begin{align*}
            \enorm{\popopmlr (\theta) - \theta^*} < 0.9 \enorm{\theta - \theta^*}.
        \end{align*}
    \end{corollary}
    The proof of Corollary~\ref{corollary:high_SNR_contraction} is in Appendix~\ref{subsec:proof:corollary:high_SNR_contraction}. Furthermore, from the uniform concentration Lemma \ref{lemma:uniform_concentration_mlr} in Appendix~\ref{sec:concentration_sample_EM}, for all $\theta: \enorm{\theta - \theta^*} \le O(\enorm{\theta^*})$, we have
    \begin{align*}
        \enorm{M_{n, mlr}(\theta) - M_{mlr}(\theta)} &\leq C \sqrt{d \log^2 (n/\delta)/n}
    \end{align*}
    with probability $1 - \delta$. From here, we can check that
    \begin{align*}
        \enorm{\theta_n^t - \theta^*} \lesssim \parenth{0.9}^t \enorm{\theta - \theta^*} + O \parenth{\sqrt{d \log^2(n/\delta) /n}}.
    \end{align*}
    
    {\bf Case (ii) $C_{0}(d \log^2(n/\delta) /n)^{1/4} \leq \enorm{\theta^*} \le 1$:} In this case, the result of Lemma~\ref{lemma:theorem4_global} shows that:
    \begin{corollary}
        \label{corollary:middle_SNR_contraction}
        When $\enorm{\theta^*} \le 1$ and $\sin \alpha < 0.1$, we have
        \begin{align}
            \enorm{M_{mlr}(\theta) - \theta^*} &\le \parenth{1 - \frac{1}{8} \enorm{\theta^*}^2} \enorm{\theta - \theta^*}.
        \end{align}
    \end{corollary}
    In order to analyze the convergence of finite-sample EM operator, we first divide the iterations into several epochs. Let $\bar{C}_0 = \enorm{\theta_n^0 - \theta^*}$. We consider that in each $l^{th}$ epoch, $\theta$ satisfies $\bar{C}_0 2^{-l-1} \le \enorm{\theta - \theta^*} \le \bar{C}_0 2^{-l}$. Note that such consideration of dividing into several epochs is only conceptual, and does not affect the implementation of the EM algorithm. 
    
    Consider we are in $l^{th}$ epoch such that $\bar{C}_0 2^{-l-1} \le \enorm{\theta - \theta^*} \le \bar{C}_0 2^{-l}$. The key idea is that in each epoch, EM makes a progress toward the ground truth as long as the improvement in population operator overcomes the statistical error, {\it i.e.,}
    $$\frac{1}{8}  \enorm{\theta^*}^2 \enorm{\theta - \theta^*} \geq 2 c r \sqrt{d \log^2(n/\delta) /n},$$ 
    where $c$ is a constant in Lemma \ref{lemma:deviation_bound_mlr}. Here, since $\enorm{\theta} \le \enorm{\theta^*} + \enorm{\theta - \theta^*}$, we can set $r = \enorm{\theta^*} + \bar{C}_0 2^{-l}$. This in turn implies that in $l^{th}$ epoch, if the following is true:
    $$\frac{1}{8} \enorm{\theta^*}^2 \bar{C}_0 2^{-l-1} \ge 2 c r \sqrt{d\log^2(n/\delta) /n} \ge 4 c (\enorm{\theta^*} + \bar{C}_0 2^{-l}) \sqrt{d \log^2(n/\delta) /n},$$ 
    then we have
    \begin{align*}
        \enorm{\samopmlr(\theta) - \theta^*} \le \parenth{1 - \frac{1}{16} \enorm{\theta^*}^2} \enorm{\theta - \theta^*}.
    \end{align*}
    Arranging the terms, we require that
    \begin{align*}
        \bar{C}_0 2^{-l} \parenth{\enorm{\theta^*}^2 - c_1 \sqrt{d \log^2(n/\delta) /n}} \ge c_2 \enorm{\theta^*} \sqrt{d \log^2 (n/\delta) /n},
    \end{align*}
    for some universal constants $c_1, c_2 > 0$. Recall that we are in middle SNR regime where (with appropriately set constants) $$\enorm{\theta^*}^2 \geq (c_1 + 1) \sqrt{d \log^2(n/\delta)/n}.$$
    Therefore, $\theta$ is guaranteed to move closer to $\theta^*$ as long as $\bar{C}_0 2^{-l} \leq c_2 \enorm{\theta^*}^{-1} \sqrt{d \log^2 (n/\delta)/n}$. Note that each epoch takes $O(\enorm{\theta^*}^{-2})$ iterations to enter the next epoch. We can conclude that after $l = O(\log(n/d))$ epochs, we enter the region where $\enorm{\theta - \theta^*} \leq c_{2} \enorm{\theta^*}^{-1} \sqrt{d \log^2 (n/\delta)/n}$ for some absolute constant $c_2 > 0$. 
    
    For $\delta$ probability bound, we can replace $\delta$ with $\delta / \log (n/d)$ and take a union bound of the uniform deviation of finite-sample EM operators given in Lemma \ref{lemma:uniform_concentration_mlr} for all epochs. This does not change the complexity in the final statistical error. 
    
    Finally, the required number of iterations in each epoch is $O(\enorm{\theta^*}^{-2})$ to make $\enorm{\theta - \thetastar}$ a half. Since the total number of epoch we require is $O(\log (n/d))$, the total number of iterations is at most $O(\enorm{\theta^*}^{-2} \log (n/d))$, concluding the proof in middle-high SNR regime.
    
    \begin{remark}
        After $O(\log(n/d))$ epochs, studying on the property of the Hessian in a very close neighborhood of $\enorm{\theta^*}$ may lead to a guarantee that EM indeed converges to the empirical MLE, see Section 6 in \cite{wu2019randomly} for example.
    \end{remark}

\subsection{Low SNR Regime}
\label{subsec:full_proof_low_SNR}
As mentioned in the main text, the core idea of the low SNR regime is that EM essentially cannot distinguish the cases between $\theta^* = 0$ and $\theta^* \neq 0$. Therefore, instead of studying the contraction of population EM operator to $\thetastar$, we study its contraction to 0. Given that insight, we have the following result with the norm of population EM operator: 
\begin{lemma}
    \label{lemma:contraction_mlr_low_SNR}
    There exists some universal constants $c_u > 0$ such that,
    \begin{align*}
        \enorm{\theta} (1 - 4 \enorm{\theta}^2 - c_u \enorm{\theta^*}^2) \le \enorm{\popopmlr(\theta)} \le \enorm{\theta} (1 - \enorm{\theta}^2 + c_u \enorm{\theta^*}^2).
    \end{align*}
\end{lemma}
The proof of the Lemma \ref{lemma:contraction_mlr_low_SNR} is in Section \ref{subsec:proof:contraction_mlr_low}. The result of Lemma~\ref{lemma:contraction_mlr_low_SNR} shows that the contraction coefficient of the population operator $\popopmlr$ consists of two terms: the non-expansive term, which is at the order of $1 - \mathcal{O}(\|\theta\|^2)$, and the quadratic term $\|\thetastar\|^2$ (up to some constant). Since we are in low SNR regime, the contraction coefficient gets close to 1. It demonstrates that the updates from population EM operator suffers from sub-linear convergence rate, instead of geometric convergence rate as that in high SNR regime.

From Lemma \ref{lemma:deviation_bound_mlr}, we immediately have that
$$\sup_{\enorm{\theta} \le r} \enorm{\samopmlr(\theta ) - \popopmlr(\theta)} \le cr\sqrt{d \log^2(n/\delta) /n},$$ 
for some universal constant $c > 0$. 

Given the contraction of population EM operator and the deviation bound between the sample and population EM operators, we are ready to study the convergence behaviors of EM algorithm under the low SNR regime. Our proof argument follows the localization argument used in Case (ii) of middle SNR regime. In particular, let the target error be $\epsilon_n := C \sqrt{d \log^2(n/\delta)/n}$ with some absolute constant $C > 0$. We assume that we start from the initialization region where $\enorm{\theta} \le \epsilon_n^{\alpha_0}$ for some $\alpha_0 \in [0, 1/2)$.

The localization argument proceeds as the following: suppose that $\epsilon_n^{\alpha_{l+1}} \le \enorm{\theta} \le \epsilon_n^{\alpha_l}$ at the $l^{th}$ epoch for $l \ge 0$. We let $C > 0$ sufficiently large such that $$\epsilon_n \ge 4 c_u \enorm{\theta^*}^2 + 4 \sup \limits_{\theta \in \ball(\theta^{*}, r_l)} \enorm{ \samopind( \theta) - \popopind ( \theta)} / r_l,$$ with $r_l = \epsilon_n^{\alpha_l}$. During this period, from Lemma \ref{lemma:contraction_mlr_low_SNR} on contraction of population EM, and Lemma \ref{lemma:deviation_bound_mlr} concentration of finite sample EM, we can check that
\begin{align*}
    \enorm{\samopind(\theta)} &\le \enorm{\theta} - 0.5 \enorm{\theta}^3 + c_u \enorm{\theta} \enorm{\theta^*}^2 + \sup \limits_{\theta \in \ball(\theta^{*}, r)} \enorm{\samopind( \theta) - \popopind ( \theta)} \\
    &\le \enorm{\theta} - \frac{1}{2} \epsilon_n^{3\alpha_{l+1}} + \frac{1}{4} \epsilon_n^{\alpha_l + 1}.
\end{align*}
Note that this inequality is valid as long as $\epsilon_n^{\alpha_{l+1}} \le \enorm{\theta} \le \epsilon_n^{\alpha_{l}}$. Now we define a sequence $\alpha_l$ using the following recursion:
\begin{align}
    \label{eq:alpha_sequence_indep}
    \alpha_{l+1} = \frac{1}{3} (\alpha_l + 1).
\end{align}
The limit point of this recursion is $1/2$, which will give $\epsilon_n^{\alpha_\infty} \approx (d/n)^{1/4}$ as argued in the main text. Hence during the $l^{th}$ epoch, we have
\begin{align*}
    \enorm{\samopind(\theta)} &\le \enorm{\theta} - \frac{1}{4} \epsilon_n^{\alpha_l + 1}.
\end{align*}
Furthermore, the number of iterations required in $l^{th}$ epoch is 
\begin{align*}
    t_l := (\epsilon_n^{\alpha_l} - \epsilon_n^{\alpha_{l+1}}) / \epsilon_n^{\alpha_l + 1} \le \epsilon_n^{-1}.
\end{align*}
After getting out of $l^{th}$ epoch, it gets into $(l+1)^{th}$ epoch which can be analyzed in the same way. From this, we can conclude that after going through $l$ epochs in total, we have $\enorm{\theta} \le \epsilon_n^{\alpha_{l+1}}$. Note that the number of EM iterations taken up to this point is $l \epsilon_n^{-1}$. 

It is easy to check $\alpha_{l} = (1/3)^l (\alpha_0 - 1/2) + 1/2$ from \eqref{eq:alpha_sequence_indep}. We can set $l = C \log (1/\beta)$ for some universal constant $C$ such that $\alpha_l$ is $1/2 - \beta$ for arbitrarily small $\beta > 0$. In conclusion, $\enorm{\theta_n^t} \le \epsilon_n^{1/2 - \beta} \leq c \cdot (d \ln^2(n/\delta)/n)^{1/4 - \beta/2}$ with high probability as long as $t \ge \epsilon_n^{-1} l \gtrsim \sqrt{d/n} \log (1/\beta)$ where $c$ is some universal constant. Hence we can set $\beta = C/ \log(d/n)$ to get a desired result $\enorm{\theta_n^t} \leq c \cdot (d \ln^2(n/\delta)/n)^{1/4}$. Since $\enorm{\theta^*} \le C_0 (d \ln^2(n/\delta) /n)^{1/4}$, it implies $\enorm{\theta_n^t - \theta^*} \leq c_{1} (d \ln^2(n/\delta)/n)^{1/4}$ where $c_{1}$ is some universal constant.

Note that we need the union bound of the concentration of sample EM operators for all $l = 1, ..., C \log (1/\beta)$, such that the argument holds for all epochs. For this purpose, we can replace $\delta$ by $\delta / \log(1/\beta)$. This does not change the order of $\epsilon_n$, hence the proof is complete.


\section{Global Convergence of the (Easy) EM}
\label{Appendix:2MLR_init_global}
This appendix gives a full proof of Theorem \ref{theorem:EM_mlr_global}. We prove the result for bounded instances with $\{\theta^*: \enorm{\theta^*} \le C\}$ for some universal constant $C > 0$. The global convergence property of the (Easy)-EM algorithm will be used for the initialization for Theorem \ref{theorem:EM_mlr_upper_bound}, hence we will focus on the iterations that the estimator stays outside of the initialization region. While we start with Easy-EM when $\cos \angle(\theta_n^0, \theta^*)$ is in order $O(1/\sqrt{d})$, note that we can safely go back to the standard EM algorithm as soon as $\cos \angle(\theta_n^t, \theta^*)$ becomes $\Theta(1)$ (see Section 4 in~\cite{kwon2019global} for more details).

\subsection{Decreasing Norm with Large Initialization in Low SNR Regime} 
\label{subsec:decrease_norm}
In low SNR regime, we require that $\enorm{\theta_n^0} \le 0.2$. Here, when we initialize with large norm such that $\enorm{\theta_n^0} \ge 0.2$, we show that in a finite number of steps it becomes that $\enorm{\theta_n^0} \le 0.2$. We remark that in low SNR regime we consider when $\enorm{\theta^*} \ll 1$.

First, suppose $\enorm{\theta} \ge 2/3$. Then,
\begin{align*}
    \enorm{\popopmlr(\theta)} &\le \sup_{u \in \sphere^{d-1}} \Exs[(X^\top \theta^*) (X^\top u) \tanh(YX^\top \theta)] + \Exs[Z (X^\top u) \tanh(YX^\top \theta)] \\
    &\le \sup_{u \in \sphere^{d-1}} \sqrt{\Exs[(X^\top \theta^*)^2] \Exs[(X^\top u)^2]} + \Exs[|Z (X^\top u)|], \\
    &\le \enorm{\theta^*} + \Exs[|Z (X^\top u)|] \le \enorm{\theta^*} + 2 / \pi.
\end{align*}
where $Z \sim \NORMAL(0,1)$ such that $Y = X^\top \theta^* + Z$. Since the uniform deviation in Easy-EM is given by Lemma \ref{lemma:uniform_concentration_mlr} as $\sqrt{d \log^2(n/\delta) / n}$, we can conclude that
\begin{align*}
    \enorm{\samopmlr(\theta)} &\le \enorm{\popopmlr(\theta)} + O \parenth{ \sqrt{d \log^2(n/\delta) / n} } \\
    &\le \enorm{\theta^*} + 2/\pi + O \parenth{\sqrt{d \log^2(n/\delta) / n}} \le 2/3.
\end{align*}

Next, suppose $0.2 \le \enorm{\theta} \le 2/3$. Following the notation in Appendix \ref{Appendix:additional_notation}, we recall equation~\eqref{eq:pop_mlr_transformed}, 
\begin{align*}
    \popopmlr(\theta) = \Exs[yx_1 \tanh(yx_1 \enorm{\theta})] v_1 + \Exs[yx_2 \tanh(yx_1 \enorm{\theta})] v_2,
\end{align*}
where $y = X^\top \theta^* + z$ where $z \sim \NORMAL(0,1)$, $x_1 = X^\top v_1$ and $x_2 = X^\top v_2$. We will see in Appendix \ref{subsec:proof:contraction_mlr_low} that $\popopmlr(\theta)^\top v_2 \le \frac{1}{2} \enorm{\theta} \enorm{\theta^*}^2 \le c_0 \sqrt{d\log^2(n/\delta) /n}$ for some absolute constant $c_0 > 0$. Therefore, we focus on bounding the first term.

Let $a = 4$, and define event $\event := \{ x_1^2 + z^2 \le a \}$. We expand $\popopmlr(\theta)$ as follows:
\begin{align*}
    \popopmlr(\theta)^\top v_1 &\le \enorm{\theta} \Exs[y^2 x_1^2 1_{\event}] +  \Exs[|yx_1| 1_{\event^c}] \\
    &\le \enorm{\theta} \Exs[z^2 x_1^2 1_{\event}] +  \Exs[|z x_1| 1_{\event^c}] + O(\enorm{\theta^*}).
\end{align*}
By converting the above expression to Rayleigh distribution with $x_1 = r \cos w, z = r \sin w$, we can more explicitly find the values of the expectations in the above equation. That is,
\begin{align*}
    \Exs[z^2 x_1^2 1_{\event}] = \frac{1}{2\pi} \int_0^{2\pi} \cos^2 w \sin^2 w dw \int_{0}^4 r^5 \exp(-r^2 / 2) dr \approx 1 - 0.013,
\end{align*}
and
\begin{align*}
    \Exs[|z x_1| 1_{\event^c}] = \frac{1}{2\pi} \int_0^{2\pi} |\cos w \sin w| dw \int_{4}^\infty r^3 \exp(-r^2 / 2) dr \le 0.002,
\end{align*}
Now using the condition that $\enorm{\theta} \le 0.2$, we have
\begin{align*}
    \popopmlr(\theta)^\top v_1 \le \enorm{\theta} (1 - 0.003) + O(\enorm{\theta^*}) \le \gamma \enorm{\theta} + O(\enorm{\theta^*}),
\end{align*}
where $\gamma = 0.997 < 1$. Since the deviation of finite-sample EM operator is in order $\sqrt{d\log^2(n/\delta)/n}$, we can conclude that
\begin{align*}
    \enorm{\popopmlr(\theta)} \le \gamma \enorm{\theta} + O \parenth{\sqrt{d \log^2(n/\delta) / n} + \enorm{\theta^*}}.
\end{align*}
Hence we can conclude that after $t = O(1)$ iterations, $\enorm{\theta_n^t} \le 0.2$.

\subsection{Angle Convergence in Middle-to-High SNR Regime}
\label{subsec:proof:global_angle}
Now we work in the regime where $\enorm{\theta^*} = \eta \ge c_\eta (d \log(n/\delta)^2 / n)^{1/4}$ for some sufficiently large constant $c_\eta > 0$. We first focus on the convergence of angle from random initialization.

Let us denote $\alpha_t := \angle (\theta_n^t, \theta^*)$. Note that since we initialize with a random vector sampled uniformly from the unit sphere, $\cos \alpha_0 = O(1/\sqrt{d})$. We bring the following lemma for the change in angles for a {\it fixed} estimator $\theta_n^t$ given in \cite{kwon2019global}:
\begin{lemma}[Theorem 8 in \cite{kwon2019global}]
    \label{lemma:theorem8_cosine_sin_global}
    Let $\epsilon_f := c_0 \max(1, \eta^{-1}) \sqrt{d / n}$ be the statistical fluctuation with some universal constant $c_0 > 0$ in one-step iteration of Easy-EM. Suppose the norm of the current estimator $\enorm{\theta_n^t}$ is larger than $\enorm{\theta^*} / 10$. Then we have,
    \begin{align}
        \cos \alpha_{t+1} &\ge \kappa_t (1 - 10 \epsilon_f) \cos \alpha_t - \frac{\epsilon_f}{\sqrt{d}}, \\
        \sin^2 \alpha_{t+1} &\le \kappa_t' \sin^2 \alpha_t +  \epsilon_f,
    \end{align}
    where $\kappa_t = \sqrt{1 + \frac{\sin^2 \alpha_t}{\cos^2 \alpha_t + \frac{1}{2} (1 + \eta^{-2})}} \ge 1$, and $\kappa_t' = \parenth{1 + \frac{2\eta^2}{1 + \eta^2} \cos^2 \alpha_t}^{-1} < 1$.
\end{lemma}
Here, the $\kappa_t$ comes from Theorem 2 in \cite{kwon2019global} for the convergence rate of the cosine values of the population EM operator. The key idea in the above lemma is that when we bound the statistical error of cosine value, we need to bound an error in one fixed direction $u := \theta^* / \enorm{\theta^*}$ instead of all directions in $\mathbb{R}^d$ to bound $l_2$ norm. More specifically, they show that
\begin{align*}
    \parenth{\frac{1}{n} \sum_i (X_i^\top u) Y_i \tanh(Y_i X_i^\top \theta) - \popopmlr(\theta)^\top u} \lesssim (1 + \enorm{\theta^*}) \sqrt{1/n} \lesssim (1 + \enorm{\theta^*}) \epsilon_f / \sqrt{d}.
\end{align*}
\begin{remark}
\label{Remark:global_angle_boundedness}
\cite{kwon2019global} requires the sample-splitting scheme in which we draw a new batch of samples at every step. The main challenge when we try to remove the sample-splitting is to show that the above argument holds for all $\theta: \enorm{\theta} \le r$ where $r = O(\max\{1, \enorm{\theta^*}\})$. For large $\enorm{\theta^*}$, getting a right order of uniform statistical error is challenging: discretization of $\theta$ results in extra $\sqrt{d}$ factor, while the Ledoux-Talagrand type approach as in Lemma \ref{lemma:uniform_concentration_mlr} results in extra $O(\enorm{\theta^*})$ factor. Therefore, here we show only for bounded instances with $\enorm{\theta^*} \le C$, and leave the analysis for arbitrarily large $\enorm{\theta^*}$ as future work.
\end{remark}

Now we adopt their approach to work {\it without} sample-splitting, and get a right order of sample complexity. First, when we work with bounded $\theta^*$, we follow the steps in Lemma \ref{lemma:uniform_concentration_mlr}, while we can skip the procedure in which we take a union bound over $1/2$-covering set of the unit sphere to bound $l_2$ norm of a random vector. This yields that
\begin{align}
    \label{eq:uniform_concentration_one_direction}
    \sup_{\enorm{\theta} \le r} \abss{\frac{1}{n} \sum_i (X_i^\top u) Y_i \tanh(Y_i X_i^\top \theta) - \popopmlr(\theta)^\top u} \le c r \sqrt{\log^2(n/\delta)/n},
\end{align}
for the absolute constant $c > 0$ given by Lemma \ref{lemma:uniform_concentration_mlr}. Let $\epsilon_f := c \sqrt{d \log^2(n/\delta)/n}$. The cosine value can be bounded as follows:
\begin{align*}
    \cos \alpha_{t+1} &= \frac{(\theta^*)^\top \theta_n^{t+1}}{\enorm{\theta_n^{t+1}}\enorm{\theta^*}} \\
    &= \frac{u^\top (\popopind(\theta_n^t) - \theta_n^{t+1})}{\enorm{\theta_n^{t+1}}} + \frac{u^\top \popopind(\theta_n^t)}{\enorm{\popopind(\theta_n^t)}} \frac{\enorm{\popopind(\theta_n^t)}}{\enorm{\theta_n^{t+1}}}, \\
    &\ge -\frac{\epsilon_f}{\sqrt{d}} \frac{r}{\enorm{\theta_n^{t+1}}} + \frac{u^\top \popopind(\theta_n^t)}{\enorm{\popopind(\theta_n^t)}} \frac{\enorm{\popopind(\theta_n^t)}}{\enorm{\popopind(\theta_n^t)} + r \epsilon_f} \\
    &\ge \kappa_t \cos \alpha_t \parenth{1 - \frac{r \epsilon_f}{\enorm{\popopind(\theta_n^t)}}} -  \frac{\epsilon_f}{\sqrt{d}} \frac{r}{\enorm{\popopmlr(\theta_n^t)} - r\epsilon_f},
\end{align*}
where the last inequality comes from Theorem 2 in \cite{kwon2019global}.

Finally, we need to show that $r/\enorm{\popopmlr(\theta_n^t)} = O(1)$ such that we can set $\epsilon_f$ as some sufficiently small absolute constant (that does not depend on $\eta$). We first need the following lemma on the norm of the next estimator:
\begin{lemma}
    \label{lemma:norm_lower_bound}
    If $\enorm{\theta} \le \enorm{\theta^*} / 10$, then 
    \begin{align*}
        \enorm{\popopmlr(\theta)} \ge \enorm{\theta} ( 1 + d_1 \cdot \min \{1,  \enorm{\theta}^2\} ).
    \end{align*}
    Otherwise, if $\enorm{\theta} \ge \enorm{\theta^*} / 10$, we have
    \begin{align*}
        \enorm{\popopmlr(\theta)} \ge \frac{\enorm{\theta^*}}{10} (1 + d_2 \cdot \min \{ 1, \enorm{\theta^*}^2 \} ). 
    \end{align*}
    for some universal constants $d_1, d_2 > 0$.
\end{lemma}
We defer the proof of this lemma to Appendix \ref{subsec:proof:lemma_norm_lower_bound}. We need the uniform concentration \eqref{eq:uniform_concentration_one_direction} for several values of $r = C_0, C_0 2^{-1}, ..., C_0 2^{-l+1}, C_0 2^{-l}$ where $C_0 = 3 C$ and $l = O(\log(n/d))$. We can replace $\delta$ by $\delta / \log(n/d)$ for union bound, which does not change the order of statistical error. Pick $k$ such that $C_0 2^{-k} \le \enorm{\theta_n^t} \le C_0 2^{-k+1} = r$. 

When $\enorm{\theta_n^t} \le \enorm{\theta^*} / 10$, we can apply the Lemma \ref{lemma:norm_lower_bound} to see
\begin{align*}
    r / \enorm{\popopmlr(\theta_n^t)} \le C_0 2^{-k + 1} / (C_0 2^{-k}) = 2,
\end{align*}
where we used $r = 2^{-k+1}$. Therefore, $r / \popopmlr(\theta_n^t) = O(1)$. On the other hand, if $\enorm{\theta_n^t} \ge \enorm{\theta^*} / 10$, then we divide the cases when $\enorm{\theta^*} \ge 1/\max(3,c_2)$ where $c_2 > 0$ satisfies the lower bound given in equation~\eqref{eq:contraction_pop_lower}:
\begin{align*}
    \enorm{\popopmlr(\theta)} \ge \enorm{\theta}(1 - 3\enorm{\theta}^2) - c_2 \enorm{\theta} \enorm{\theta^*}^2.
\end{align*}

When $\enorm{\theta^*} \ge 1/\max(3,c_2)$ and $\enorm{\theta_n^t} \ge \enorm{\theta^*} / 10$, by Lemma \ref{lemma:norm_lower_bound} we have $r/\popopmlr(\theta) \le C_0 \max(3, c_2) = O(1)$ since all parameters here are universal constants. On the other hand, if $\enorm{\theta^*} \le 1/\max(3,c_2)$ and $\enorm{\theta_n^t} \ge \enorm{\theta^*} / 10$, then from equation~\eqref{eq:contraction_pop_lower} we have
\begin{align*}
    \enorm{\popopmlr(\theta)} \ge \enorm{\theta}(1 - 3\enorm{\theta}^2) - c_2 \enorm{\theta} \enorm{\theta^*}^2 \ge \enorm{\theta}/2.
\end{align*}
Therefore, $r / \enorm{\popopmlr(\theta_n^t)} \le C_0 2^{-k+1} / (C_0 2^{-k-1}) = 4 = O(1)$. 

From the above case study, we have that
\begin{align*}
    \cos \alpha_{t+1} \ge \kappa_t \cos \alpha_t (1 - c_4 \epsilon_f) - c_5 \frac{\epsilon_f}{\sqrt{d}},
\end{align*}
for some absolute constants $c_4, c_5 > 0$. Now observe that as long as $\sin \alpha_{t} > c_\alpha$, $\kappa_t = 1 + c_6 \min \{1, \eta^{2}\}$ for some sufficiently small constant $c_\alpha, c_6 > 0$. Also, recall that we are considering the middle-to-high SNR regime when $\eta^2 \ge c_\eta \sqrt{d \log^2(n/\delta) /n}$ for some sufficiently large constant $c_\eta > 0$, whereas $\epsilon_f \le c \sqrt{d \log^2(n/\delta)/n}$ for another fixed constant $c > 0$. Therefore, there exists a universal constant $c_7 > 0$ such that for all $\cos \alpha_t \ge 1/\sqrt{d}$, we have 
\begin{align*}
    \cos \alpha_{t+1} \ge (1 + c_7 \min(1, \eta^2)) \cos \alpha_t.
\end{align*}
After $t = O(\eta^{-2} \log(d))$ iterations starting from $\cos \alpha_0 = 1/\sqrt{d}$, we have $\cos \alpha_{t} \ge 0.95$ or $\sin \alpha_{t} \le 0.1$.

\subsection{Stability and Convergence in Middle-to-High SNR Regime after Alignment}
\label{subsec:proof:global_stability_convergence}
In this subsection, we see how the alignment is stabilized and the norm increases in case we start from small initialization.

\paragraph{Sine stays below some threshold.} Once $\theta_n^t$ and $\theta^*$ are well-aligned, using $\sin^2 \alpha_t = 1 - \cos^2\alpha_t$, similar arguments can be applied for $\sin$ values:
\begin{align*}
    \sin^2 \alpha_{t+1} &\le (1 - c_1 \min(1, \eta^2)) \sin^2 \alpha_t, \quad &\text{if} \quad \sin^2 \alpha_t \ge c_2 \\
    \sin^2 \alpha_{t+1} &\le c_2, \quad &\text{else} \quad \sin^2 \alpha_t \le c_2,
\end{align*}
for some absolute constants $c_1 > 0$ and sufficiently small $0 < c_2 < 0.01$ given that $\cos \alpha_t > 0.95$.

\paragraph{Initialization from small estimators after alignment.} After the angle is aligned such that $\sin \alpha_t \le c_2$. We see how fast $\enorm{\theta_n^t}$ enters the desired initialization region that Theorem \ref{theorem:EM_mlr_upper_bound} requires, when $\enorm{\theta_n^t} \le 0.9\enorm{\theta^*}$. 

Let us first consider the case $0.1 \enorm{\theta^*} \le \enorm{\theta_n^t} \le 0.9\enorm{\theta^*}$. We recall Lemma \ref{lemma:theorem4_global} such that
\begin{align*}
    \enorm{\theta^*-\popopmlr(\theta_n^t)} &\le \kappa \enorm{\theta_n^t - \theta^*} + \kappa 16 \sin^2\alpha \enorm{\theta_n^t - \theta^*} \frac{\eta^2}{1 + \eta^2} \\
    &\le \kappa (1 + (16 \sin^2 \alpha) \eta^2) \enorm{\theta_n^t - \theta^*},
\end{align*}
where $\kappa < 1 - c_3 \eta^2$ for some absolute constant $c_3$. By appropriately setting $c_2$ and $c_3$, we have  
\begin{align*}
    \enorm{\theta^*-\popopmlr(\theta_n^t)} \le (1 - c_4 \min(1, \eta^2)) \enorm{\theta - \theta^*}, 
\end{align*}
for some constant $c_4 > 0$. Since we are in the regime $\eta^2 \ge c_\eta \sqrt{d \log^2(n/\delta) / n}$ for sufficiently large $c_\eta$, by appropriately setting the constants we have $\enorm{\samopmlr(\theta_n^t) - \theta^*} \le (1 - c_5 \min(1, \eta^2)) \enorm{\theta - \theta^*}$ for some absolute constant $c_5 > 0$, as long as we are in the region $0.1 \enorm{\theta^*} \le \enorm{\theta_n^t} \le 0.9 \enorm{\theta^*}$. Hence after $O(\max(1, \eta^{-2}))$ iterations, we reach to the desired initialization region. 

Now we consider the case $\enorm{\theta} \le 0.1 \enorm{\theta^*}$. In this case, by Lemma \ref{lemma:norm_lower_bound}, we can show that
\begin{align*}
    \enorm{\popopmlr(\theta)} \ge \enorm{\theta} (1 + c_6 \min\{ 1, \enorm{\theta}^2, \enorm{\theta^*}^2 \}),
\end{align*}
for some universal constant $c_6 > 0$. After $O(\max \{ \enorm{\theta}^{-2}, \enorm{\theta^*}^{-2} \})$ iterations, we enter $\enorm{\theta} \ge \enorm{\theta^*} / 10$. Note that when we start with $\enorm{\theta_n^0} = \Omega(1)$, $\enorm{\theta_n^t}$ will stay above $\min \{\Omega(1), \enorm{\theta^*} / 10\}$ throughout all iterations due to Lemma \ref{lemma:norm_lower_bound} and Lemma \ref{eq:uniform_deviation_mlr}.


\section{Deferred Lemmas}
\label{appendix:convergence_population_EM}
In this appendix, we collect proofs for auxiliary lemmas which were postponed in the proof of main theorems: the contraction of population EM operators under both middle and low SNR regimes, uniform deviation of finite-sample EM operators, and the lower bounds on the norms of population EM operators.

\subsection{Contraction of the Population EM Operator under Low SNR Regime}
\label{subsec:contraction_population_low_SNR}
\subsubsection{Proof of Lemma \ref{lemma:contraction_mlr_low_SNR}}
\label{subsec:proof:contraction_mlr_low}

We use notations and definitions stated in \ref{Appendix:additional_notation}. 
    
    \paragraph{Upper Bound:} We first bound the first coordinate of the population operator from equation~\eqref{eq:pop_mlr_transformed}:
    \begin{align*}
        M_{mlr}(\theta)^\top v_1 &= \Exs_{x_1, x_2, y} \brackets{\tanh(y x_1 \enorm{\theta}) x_1 y},
    \end{align*}
    We will expand the above equation using Taylor series bound of $x \tanh(x)$:
    \begin{align}
        x^2 - \frac{x^4}{3} \leq x \text{tanh}(x) \leq x^2 - \frac{x^4}{3} + \frac{2 x^6}{15}. \label{eq:lemma_contraction_pop_ind_2}
    \end{align}
    Now we unfold the equation above, we have
    \begin{align*}
        M_{mlr}(\theta)^\top v_1 &= \frac{1}{\enorm{\theta}} \Exs_{x_1, x_2, y}  \brackets{\tanh(y x_1 \enorm{\theta}) y x_1 \enorm{\theta}} \\
        &\le \frac{1}{\enorm{\theta}} \Exs_{x_1, x_2, y}  \brackets{(y x_1 \enorm{\theta})^2 - \frac{(y x_1 \enorm{\theta})^4}{3} + \frac{2(y x_1 \enorm{\theta})^6}{15}} \\
        &\le \frac{1}{\enorm{\theta}} \Exs_{x_1, z}  \biggr[ (x_1 \enorm{\theta} (z + x_1 b_1^* + x_2 b_2^*))^2 - \frac{(x_1 \enorm{\theta} (z + x_1 b_1^* + x_2 b_2^*))^4}{3} \\
        & \qquad \qquad \qquad + \frac{2(x_1 \enorm{\theta} (z + x_1 b_1^* + x_2 b_2^*))^6}{15} \biggr],
    \end{align*}
    where $z \sim \NORMAL(0, 1)$. Note here that, any (constantly) higher order terms of Gaussian distribution is constant. Hence instead of computing all coefficients explicitly for all monomials, we can simplify the argument as 
    \begin{align}
        \label{eq:pop_mlr_firstvalue_taylor}
        M_{mlr}(\theta)^\top v_1 &\le \frac{1}{\enorm{\theta}} \Exs_{x_1, z}  \biggr[ (x_1 \enorm{\theta} z)^2 - \frac{(x_1 \enorm{\theta} z)^4}{3} + \frac{2(x_1 \enorm{\theta} z)^6}{15} \biggr] + c_1 \enorm{\theta} \enorm{\theta^*}^2, \nonumber \\
        &= \enorm{\theta} (1 - 3 \enorm{\theta}^2 + 30 \enorm{\theta}^4) + c_1 \enorm{\theta} \enorm{\theta^*}^2,
    \end{align}
    for some universal constant $c_1 > 0$. Since we assumed $\enorm{\theta} < 0.2$, we have $3\enorm{\theta}^2 - 30\enorm{\theta}^4 \ge \enorm{\theta}^2$. We conclude that 
    \begin{align*}
        M_{mlr}(\theta)^\top v_1 \le \enorm{\theta} (1 - \enorm{\theta}^2 + c_1 \enorm{\theta^*}^2).
    \end{align*}
    
    Then we bound the value in the second coordinate of the population operator:
    \begin{align*}
        M_{mlr}(\theta)^\top v_2 &= \Exs_{x_1, x_2, y} \brackets{\tanh(y x_1 \enorm{\theta}) y x_2},
    \end{align*}
    where $y|(x_1, x_2) \sim \NORMAL(x_1 b_1^* + x_2 b_2^*, 1)$. In order to derive an upper bound for the above equation, we rely on the following equation which we defer the proof to the end of this section:
    \begin{align}
        \label{eq:v2_stein_lemma}
        \Exs \brackets{\tanh(y x_1 \enorm{\theta}) y x_2} = b_2^* \ \Exs \brackets{x_1^2 \tanh(x_1 \enorm{\theta} (z + x_1 b_1^*)) - \enorm{\theta} b_1^* x_1^2 \tanh' (x_1 \enorm{\theta} (z + x_1 b_1^*))},
    \end{align}
    where $z \sim \NORMAL(0, 1 + {b_2^*}^2)$ with subsuming $x_2$ from the equation. From \eqref{eq:v2_stein_lemma}, we can check that
    \begin{align*}
        \Exs \brackets{\tanh(y x_1 \enorm{\theta}) y x_2} &\le b_2^* \ \Exs \brackets{x_1^2 \tanh(x_1 \enorm{\theta} (z + x_1 b_1^*))} \\
        &= \frac{b_2^*}{2} \Exs \brackets{x_1^2 \tanh(x_1 \enorm{\theta} (z + x_1 b_1^*)) + x_1^2 \tanh(x_1 \enorm{\theta} (-z + x_1 b_1^*))} \\
        &\le b_2^* \Exs \brackets{x_1^2 \tanh(x_1^2 \enorm{\theta} b_1^*)}, \\
        &\le \enorm{\theta} b_1^* b_2^* \Exs \brackets{x_1^4} \le \frac{1}{2} \enorm{\theta} \enorm{\theta^*}^2,
    \end{align*}
    where we used $\tanh(a + x) + \tanh(a - x) \le 2 \tanh(a)$ for any $a > 0$ and $x \in \mathbb{R}$. 
    
    From the above results, we have shown that 
    \begin{align}
        \enorm{M_{mlr} (\theta)} \le |\popopmlr(\theta)^\top v_1| + |\popopmlr(\theta)^\top v_2| \le \enorm{\theta} \parenth{1 - \enorm{\theta}^2 + c \enorm{\theta^*}^2},
    \end{align} 
    for some universal constant $c > 0$. 

    \paragraph{Lower Bound:} To prove the lower bound of the population EM operator, we again expand the equation using Taylor series \eqref{eq:lemma_contraction_pop_ind_2}:
    \begin{align}
        \label{eq:contraction_pop_lower}
        \enorm{\popopmlr(\theta)} \ge |\popopmlr(\theta)^\top v_1| \ge \enorm{\theta} (1 - 3\enorm{\theta}^2) - c_2 \enorm{\theta} \enorm{\theta^*}^2.
    \end{align}
    The result follows immediately with some absolute constant $c_2 > 0$.

    \paragraph{Proof of equation \eqref{eq:v2_stein_lemma}:} 
    For the left hand side, we apply the Stein's lemma with respect to $x_2$. It gives that
    \begin{align*}
        \Exs [\tanh(\enorm{\theta} x_1 y) y x_2] &= \Exs \brackets{\frac{d}{d x_2} \tanh(\enorm{\theta} x_1 y) y} \\
        &= \Exs \brackets{\frac{d}{d x_2} \tanh(\enorm{\theta} x_1 (\bar{z} + x_1 b_1^* + x_2 b_2^*)) (\bar{z} + x_1 b_1^* + x_2 b_2^*)} \\
        &= \Exs [ b_2^* \tanh(\enorm{\theta} x_1 (\bar{z} + x_1 b_1^* + x_2 b_2^*)) \\
        &\qquad + (\enorm{\theta} x_1 b_2^*) (\bar{z} + x_1 b_1^* + x_2 b_2^*) \tanh'(\enorm{\theta} x_1 (\bar{z} + x_1 b_1^* + x_2 b_2^*)] \\
        &= b_2^* \ \Exs [\tanh(\enorm{\theta} x_1 (z + x_1 b_1^*)) + \enorm{\theta} x_1 (z + x_1 b_1^*) \tanh'(\enorm{\theta} x_1 (z + x_1 b_1^*)))] 
    \end{align*}
    where $\bar{z} \sim \NORMAL(0, 1)$ and $z \sim \NORMAL(0, 1 + {b_2^*}^2)$. For the right hand side, we apply the Stein's lemma with respect to $x_1$. First, we check the first term in the right hand side that
    \begin{align*}
        & \hspace{-4 em} \Exs [x_1^2 \tanh(\enorm{\theta} x_1 (z + x_1 b_1^*))] \\
        &= \Exs \brackets{\frac{d}{d x_1} (x_1 \tanh(\enorm{\theta} x_1 (z + x_1 b_1^*)) )} \\
        &= \Exs \brackets{\tanh(\enorm{\theta} x_1 (z + x_1 b_1^*)) + x_1 \frac{d}{dx_1} \tanh(\enorm{\theta} x_1 (z + x_1 b_1^*)} \\
        &= \Exs \brackets{\tanh(\enorm{\theta} x_1 (z + x_1 b_1^*)) + \enorm{\theta} x_1 (z + 2 x_1 b_1^*) \tanh'(\enorm{\theta} x_1 (z + x_1 b_1^*)}. 
    \end{align*}
    Plugging this into \eqref{eq:v2_stein_lemma} and subtracting the remaining term gives the result that matches to the left hand side.
\subsection{Contraction of the Population EM Operator under Middle SNR Regime}
In this appendix, we provide the proofs for contraction of the population EM operator under middle SNR regime.
\subsubsection{Proof of Corollary~\ref{corollary:high_SNR_contraction}}
\label{subsec:proof:corollary:high_SNR_contraction}
In Lemma \ref{lemma:theorem4_global}, note that $\kappa \le 1 - \frac{1}{2} \min \{ \enorm{\theta}^2, \frac{\enorm{\theta^*}^2}{\enorm{\theta^*}^2 + 1} \}$ and $(\enorm{\theta^*} \sin \alpha) < \enorm{\theta - \theta^*}$ where $\sin \alpha < 1/10$. Therefore, whenever $\enorm{\theta^*} \ge 1$, with the initialization condition $\enorm{\theta} \ge 0.9 \enorm{\theta^*}$
\begin{align*}
    \enorm{\popopmlr (\theta) - \theta^*} &\le \parenth{1 - 1/4} \enorm{\theta - \theta^*} + \kappa 16 ( \sin^2 \alpha ) \enorm{\theta - \theta^*} \le 0.9\enorm{\theta - \theta^*},
\end{align*}
which completes the proof.

\subsubsection{Proof of Corollary~\ref{corollary:middle_SNR_contraction}}
\label{subsec:proof:corollary:middle_SNR_contraction}
From Lemma \ref{lemma:theorem4_global}, note that $\frac{\eta^2}{1 + \eta^2} \le \eta^2 = \enorm{\theta^*}^2$. Using $\kappa \le 1 - \frac{1}{2} \min \{ \enorm{\theta}^2, \frac{\enorm{\theta^*}^2}{\enorm{\theta^*}^2 + 1} \}$, $(\enorm{\theta^*} \sin \alpha) < \enorm{\theta - \theta^*}$ and $\sin \alpha < 1/10$. With the initialization condition $\enorm{\theta} \ge 0.9 \enorm{\theta^*}$, we have
    \begin{align*}
        \enorm{M_{mlr}(\theta) - \theta^*} &\le \parenth{1 - \frac{1}{4} \enorm{\theta^*}^2} \enorm{\theta - \theta^*} + \frac{1}{8} \enorm{\theta^*}^2 \enorm{\theta - \theta^*} \le \parenth{1 - \frac{1}{8} \enorm{\theta^*}^2} \enorm{\theta - \theta^*}.
    \end{align*}
\subsection{Uniform deviation of finite-sample EM operator: Proof of Lemma~\ref{lemma:deviation_bound_mlr}}
\label{subsec:proof:lemma:deviation_bound_mlr}

\begin{proof}
    Let us assume that $n \ge Cd$ for sufficiently large constant $C > 0$. To simplify the notation, we use $\hat{\Sigma}_n = \frac{1}{n}\sum_i X_iX_i^\top$. Observe that
    \begin{align*}
        \enorm{\samopmlr(\theta) - \popopmlr(\theta)} &\le \opnorm{\hat{\Sigma}_n^{-1}} \enorm{\frac{1}{n} \sum_{i=1}^n Y_i X_i \tanh(Y_i X_i^\top \theta) - \popopmlr(\theta)} \\
        &\qquad + \opnorm{\hat{\Sigma}_n^{-1} - I} \enorm{\popopmlr(\theta)}.
    \end{align*}
    The first term can be bounded by $c_1 r \sqrt{d\log^2(n/\delta)/n}$ with some absolute constant $c_1 > 0$ using the results of \eqref{eq:uniform_deviation_mlr} and Lemma \ref{lemma:concentration_subexp} in Appendix \ref{sec:concentration_sample_EM}.
    
    For the second term, we first know from Lemma \ref{lemma:concentration_subexp} that $\opnorm{\hat{\Sigma}_n^{-1} - I} = \opnorm{\hat{\Sigma}_n^{-1}} \opnorm{\hat{\Sigma}_n - I} \le c_2 \sqrt{d/n})$ for some universal constant $c_2 > 0$. If we can show that $\enorm{\popopmlr(\theta)} \le O(r)$, then we are done. To see this, first we check that
    \begin{align*}
        \enorm{\popopmlr(\theta)} = \enorm{\Exs[YX \tanh(YX^\top \theta)]} \le \enorm{\theta} \opnorm{\Exs[Y^2 XX^\top]}.
    \end{align*}
    It is easy to check that $\Exs[Y^2 XX^\top] = I + 2 \theta^* {\theta^*}^\top$, hence $\opnorm{\Exs[Y^2 XX^\top]} = 1 + 2\enorm{\theta^*}^2 \le 1 + 2C^2 = O(1)$. Therefore, $\enorm{\popopmlr(\theta)} \le c_3 \enorm{\theta} \le c_3 r$ with a constant $c_3 = (1 + 2C^2)$. This completes the proof of Lemma \ref{lemma:deviation_bound_mlr}.
\end{proof}

\subsection{Lower Bound on the Norm: Proof of Lemma \ref{lemma:norm_lower_bound}}
\label{subsec:proof:lemma_norm_lower_bound}
This Lemma is in fact a more refined statement of Lemma 23 in \cite{kwon2019global} where they give a lower bound on the norms for the same purpose. We give a more refined result here.

Let $\alpha = \angle(\theta, \theta^*)$. We use the notations defined in Appendix \ref{Appendix:additional_notation}. We recall here that $b_1^* = \theta^* \cos \alpha$, $b_2^* = \theta^* \sin \alpha$. We consider three cases as in \cite{kwon2019global}. 

{\it Case (i): } $\cos \alpha \le 0.2$. This case we essentially give a norm bound for $\cos \alpha = 0$. Suppose that $\enorm{\theta} \le \enorm{\theta^*} / 10$. We can first check that
\begin{align*}
    \enorm{\popopmlr(\theta)} &\ge |\popopmlr(\theta)^\top v_1| = \Exs_{x_1, x_2, y}[\tanh(yx_1 \enorm{\theta}) yx_1] \\
    &= \Exs_{x_1, x_2, z} [\tanh((x_1 b_1^* + x_2 b_2^* + z)x_1 \enorm{\theta}) (x_1 b_1^* + x_2 b_2^* + z) x_1],
\end{align*}
where $x_1, x_2, z \sim \NORMAL(0,1)$. From the argument in \cite{kwon2019global}, the above quantity is larger than the following $b_1^* = 0$ case (see Lemma 23 in \cite{kwon2019global} for details):
\begin{align*}
    \Exs_{x_1, x_2, z} [\tanh((x_2 b_2^* + z)x_1 \enorm{\theta}) (x_2 b_2^* + z) x_1] = \Exs_{x_1, \bar{z}} [\tanh(\bar{z} x_1 \enorm{\theta}) \bar{z} x_1],
\end{align*}
where $\bar{z} \sim \NORMAL(0, 1 + (b_2^*)^2) = \NORMAL(0, \sigma_2^2)$. We can lower bound the following quantity such that
\begin{align*}
    \Exs_{x_1, \bar{z}} [\tanh(\bar{z} x_1 \enorm{\theta}) \bar{z} x_1] &\ge \sigma_2 \Exs_{x_1, z} [\tanh(\sigma_2 z x_1 \enorm{\theta}) z x_1] \\
    &\ge \sigma_2 \Exs_{x_1, z} [\tanh(z x_1 \enorm{\theta}) z x_1].
\end{align*}
If $\enorm{\theta} > 0.5$, then through the numerical integration we can check that $\Exs_{x_1, z} [\tanh(0.5 z x_1) z x_1] > 1/\pi$. Hence, we immediately have that 
$$|\popopmlr(\theta)^\top v_1| \ge \frac{1}{\pi} \sigma_2 \ge \frac{\sin \alpha}{\pi} \enorm{\theta^*} \ge \frac{1}{5} \enorm{\theta^*},$$
since $\sin \alpha > 0.9$ in this case. Since we are considering the case when $\enorm{\theta} \le \enorm{\theta^*} / 10$, clearly we have
\begin{align*}
    \enorm{\popopmlr(\theta)} \ge \enorm{\theta} (1 + 1 \cdot \min(1, \enorm{\theta}^2)). 
\end{align*}

If $\enorm{\theta} < 0.5$, then we get a lower bound using Taylor expansion:
\begin{align*}
    \Exs_{x_1, \bar{z}} [\tanh(\bar{z} x_1 \enorm{\theta}) \bar{z} x_1] &\ge \sigma_2 \parenth{\Exs_{x_1, z} [\enorm{\theta} (z x_1)^2] - \frac{1}{3} \Exs_{x_1, z} [\enorm{\theta}^3 (z x_1)^4]} \\
    &= \sigma_2 \enorm{\theta} (1 - 3 \enorm{\theta}^2) = \enorm{\theta} \sqrt{1 + 0.96 \eta^2} (1 - 3 \enorm{\theta}^2),
\end{align*}
where $\enorm{\theta^*} = \eta$. Here, we consider three cases when $\eta \ge 5$, $5 \ge \eta \ge 1$, $1 \ge \eta$. When $\eta \ge 5$, then we immediately have $|\popopmlr(\theta)^\top v_1| \ge 1.25 \enorm{\theta}$. In case $5 \ge \eta \ge 1$, we first note that since $\enorm{\theta} \le \enorm{\theta^*} / 10$, we check the value of 
$$\enorm{\theta} \sqrt{1 + 0.96 \eta^2} (1 - 0.03 \eta^2).$$
We can again, numerically check that $\sqrt{1 + 0.96 \eta^2} (1 - 0.03 \eta^2) \le 1.25$ for $1 \le \eta \le 5$. Finally, when $\eta \le 1$, then a simple algebra shows that 
\begin{align*}
    \enorm{\theta} \sqrt{1 + 0.96\eta^2} (1 - 0.03\eta^2) \ge \enorm{\theta} ( 1 + 0.3 \eta^2 ). 
\end{align*}
Combining all, we can conclude that when $\enorm{\theta} \le \frac{\enorm{\theta^*}}{10}$
\begin{align*}
    \enorm{\popopmlr(\theta)} \ge \enorm{\theta}(1 + 0.25 \cdot \min(1, \enorm{\theta^*}^2)) \ge \enorm{\theta}(1 + 0.25 \cdot \min(1, \enorm{\theta}^2)).
\end{align*}

Now note that $\popopmlr(\theta)^\top v_1$ increases in $\enorm{\theta}$, hence for all $\enorm{\theta} \ge \enorm{\theta^*} / 10$, it holds that
\begin{align*}
    \enorm{\popopmlr(\theta)} \ge \frac{\enorm{\theta^*}}{10} (1 + 0.25 \cdot \min(1, \enorm{\theta^*}^2)).
\end{align*}

{\it Case (ii):} $\cos \alpha \ge 0.2$. Again, we can only consider when $\enorm{\theta} \le \enorm{\theta^*} / 10$ since the other case will immediately follow. Their claim in this case is that $|\popopmlr(\theta)^\top v_1| \ge \min \parenth{\sigma_2^2 \enorm{\theta}, b_1^*}$. Hence we consider two cases when $\sigma_2^2 \enorm{\theta} = (1 + \eta^2 \sin^2 \alpha) \enorm{\theta} \le b_1^* = \enorm{\theta^*} \cos \alpha$ and the other case. 

In the first case when $\sigma_2^2 \enorm{\theta} \le b_1^*$, it can be shown that (see equation (50) in \cite{kwon2019global} for details)
\begin{align*}
    b_1^* - \popopmlr(\theta)^\top v_1 \le \kappa^3 (b_1^* - \sigma_2^2 \enorm{\theta}),
\end{align*}
where $\kappa \le \sqrt{1 + b_1^2}^{-1}$. Rearranging this inequality, we have
\begin{align*}
    \popopmlr(\theta)^\top v_1 &\ge \enorm{\theta^*} (1 - \kappa^3) \cos\alpha + \kappa^3 (1 + \eta^2 \sin^2 \alpha) \enorm{\theta} \\
    &\ge \enorm{\theta} 2 (1 - \kappa^3) + \kappa^3 (1 + \eta^2 \sin^2 \alpha) \enorm{\theta} \\
    &\ge \enorm{\theta} + (1 - \kappa^3) \enorm{\theta}.
\end{align*}
Note that $1 - \kappa^3 \ge c_1 \min(1, b_1^2)$ for some constant $c_1 > 0$. On the other side, if $\sigma_2^2 \enorm{\theta} \ge b_1^*$, then we immediately have 
\begin{align*}
    \popopmlr(\theta)^\top v_1 \ge \enorm{\theta^*} / 5 \ge \frac{\enorm{\theta^*}}{10} ( 1 + 1 \cdot \min(1, \enorm{\theta^*}^2)) \ge \enorm{\theta} ( 1 + 1 \cdot \min(1, \enorm{\theta}^2)).
\end{align*}
Combining two cases, we have that
\begin{align*}
    \enorm{\popopmlr(\theta)} \ge \enorm{\theta} (1 + c_1 \cdot \min (1, \enorm{\theta}^2)). 
\end{align*}

Now similarly to {\it Case (i)}, since $\popopmlr(\theta)^\top v_1$ is increasing in $\enorm{\theta}$, when $\enorm{\theta} \ge \enorm{\theta^*} / 10$, we have
\begin{align*}
    \enorm{\popopmlr(\theta)} \ge \frac{\enorm{\theta^*}}{10} (1 + c_2 \cdot \min (1, \enorm{\theta^*}^2)), 
\end{align*}
where $c_2 = c_1 / 100$. 

Collecting all results in two cases, we have Lemma \ref{lemma:norm_lower_bound}.

\section{Concentration of Measures in Finite-Sample EM}
\label{sec:concentration_sample_EM}
In all lemmas that follow, we assume that $n \ge Cd$ for sufficiently large constant $C > 0$, such that the tail probability of the sum of $n$ independent sub-exponential random variables are in sub-Gaussian decaying rate.
\begin{lemma}
    \label{lemma:concentration_subexp}
    Suppose $X \sim \NORMAL(0, I)$ and $Y|X \sim \frac{1}{2}\NORMAL(X^\top \theta^*, 1) + \frac{1}{2}\NORMAL(-X^\top \theta^*, 1)$. Then, with probability at least $1 - \delta$,
    \begin{align}
        \frac{1}{n} \sum_{i=1}^n Y_i^2 - 1  = O \parenth{(\enorm{\theta^*} + 1)^2 \sqrt{\frac{\ln(1/\delta)}{n}}}, \\
        \opnorm{\frac{1}{n} \sum_{i=1}^n X_iX_i^\top - I} = O \parenth{\sqrt{\frac{d \ln(1/\delta)}{n}}}.
    \end{align}
\end{lemma}
The above lemma is standard concentration lemmas for standard Gaussian distributions. 
\begin{lemma}
    \label{lemma:concentration_higher_order}
    Let $X, Y$ be the random variables as in Lemma \ref{lemma:concentration_subexp}. With probability at least $1 - \delta$, we have
    \begin{gather}
        \opnorm{\frac{1}{n} \sum_{i=1}^n Y_i^2 X_i X_i^\top - I} = O \parenth{(\enorm{\theta^*} + 1)^2 \sqrt{\frac{d \ln^2 (n/\delta)}{n}}},
    \end{gather}
\end{lemma}
\begin{proof}
    Let $\nu_i$ be an independent Rademacher variable and $Z_i = \NORMAL(0,1)$. We can write $Y_i = \nu_i X_i^\top \theta^* + Z_i$. We use the truncation argument for the of concentration of higher order moments. First define the good event $\event := \{\forall i \in [n], |Z_i| \le \tau, |X_i^\top \theta^*| \le \tau_2| \}$. We will decide the order of $\tau$ later such that $P(\event) \ge 1 - \delta$. Let $\widetilde{Y} \sim Y|\event, \widetilde{X} \sim X | \event$ and $(\widetilde{Y}_i, \widetilde{X}_i)$ be independent samples of $(\widetilde{Y}, \widetilde{X})$. It is easy to check that $\widetilde{Y} \widetilde{X}$ is a sub-Gaussian vector with Orlicz norm $O(\tau + \tau_2)$ \cite{Vershynin_2011}. To see this,
    \begin{align}
        \label{eq:conditioned_orcliz_norm}
        \vecnorm{\widetilde{Y} \widetilde{X}}{\psi_2} &= \sup_{u \in \sphere^{d-1}} \sup_{p \ge 1} p^{-1/2} \Exs \brackets{|Y(X^\top u)|^p|\event}^{1/p} \\
        &\le (\tau + \tau_2) \sup_{u \in \sphere^{d-1}} \sup_{p \ge 1} p^{-1/2} \Exs \brackets{|X^\top u|^p 1_{\event} }^{1/p} / P(\event)^{1/p} \\
        &\le (\tau + \tau_2) K,
    \end{align}
    for some universal constant $K>0$ and the last inequality comes from the $p^{th}$ moments of Gaussian is $O((2p)^{p/2})$ and $P(\event) \ge 1 - \delta$.
    
    Now we decompose the probability as the following:
    \begin{align*}
        \Prob \parenth{\opnorm{\frac{1}{n}\sum_{i=1}^n Y_i^2 X_i X_i^\top - I} \ge t} &\le \Prob \parenth{\opnorm{\frac{1}{n}\sum_{i=1}^n Y_i^2 X_i X_i^\top - I} \ge t | \event} + \Prob(\event^c) \\
        &\le \underbrace{\Prob \parenth{ \opnorm{\frac{1}{n}\sum_{i=1}^n \widetilde{Y}_i^2 \widetilde{X}_i \widetilde{X}_i^\top - \Exs[\widetilde{Y}^2 \widetilde{X} \widetilde{X}^\top]} \ge t/2}}_{(a)} \\
        &\ \ + \underbrace{\Prob \parenth{ \opnorm{\Exs[\widetilde{Y}^2 \widetilde{X} \widetilde{X}^\top] - I} \ge t/2}}_{(b)} + \underbrace{\Prob(\event^c)}_{(c)}.
    \end{align*}
    We can use a measure of concentration for random matrices for (a) given that $n \ge Cd$ for sufficiently large $C > 0$ \cite{Vershynin_2011}, and bound by $\exp \parenth{-\frac{nt^2}{C (\tau+\tau_2)^4} + C' d}$ for some constants $C, C' > 0$. The bound for (c) is given by $n \exp(-\tau^2)$, hence we set
    $$\tau = \Theta \parenth{\sqrt{\log (n/\delta)}}, \tau_2 = \enorm{\theta^*} \tau.$$ 
    Finally, for (b), we first note that
    \begin{align*}
        \Exs[Y^2 X X^\top] = \Exs[\widetilde{Y}^2 \widetilde{X} \widetilde{X}^\top]P(\event) + \Exs[Y^2 X X^\top 1_{\event^c}].
    \end{align*}
    Rearranging the terms, 
    \begin{align*}
        \opnorm{\Exs[\widetilde{Y}^2 \widetilde{X} \widetilde{X}^\top] - I} &\le \opnorm{\Exs[\widetilde{Y}^2 \widetilde{X} \widetilde{X}^\top]} P(\event^c) + \sqrt{\sup_{u \in \sphere^d} \Exs[Y^4 (X^\top u)^4]} \sqrt{P(\event^c)} \\
        &\le (\tau + \tau_2)^2 n \exp(-\tau^2/2) + 3 (\tau + \tau_2)^2 \sqrt{n} \exp(-\tau^2/4) \le \sqrt{1/n}.
    \end{align*}
    We can set $t = O \parenth{(\enorm{\theta^*} + 1)^2 \sqrt{d \log^2 (n/\delta) / n}}$ and get the desired result. 
\end{proof}

\begin{lemma}
\label{lemma:uniform_concentration_mlr}
    Let $X, Y$ be the random variables as in Lemma \ref{lemma:concentration_subexp}. Suppose $\enorm{\theta^*} \le C$ for some universal constant $C > 0$. Then for any given $r > 0$, with probability at least $1 - \delta$, we have
    \begin{align}
        \sup_{\theta: \enorm{\theta} \le r} \left\| \frac{1}{n} \sum_{i=1}^n Y_i X_i \tanh \parenth{Y_i X_i^\top \theta} - \popopmlr(\theta) \right\| \le c r\sqrt{\frac{d \ln^2 (n/\delta)}{n}},
    \end{align}
    for some universal constant $c > 0$.
\end{lemma}
\begin{proof}
    We start with the standard discretization argument for bounding the concentration of measures in $l_2$ norm. Let $Z(\theta) := \frac{1}{n} \sum_{i=1}^n Y_i X_i \tanh \parenth{Y_i X_i^\top \theta} - \popopmlr(\theta)$. The standard symmetrization argument gives that \cite{vanderVaart-Wellner-96, wainwright2019high}.
    
    \begin{align}
        \label{eq:symmetrization_uniform}
        \Prob \biggr( \sup \limits_{\enorm{\theta} \le r} \enorm{Z(\theta)} \ge t \biggr) \le 2 \Prob \parenth{\sup \limits_{\enorm{\theta} \le r} \left\|\frac{1}{n} \sum_{i=1}^n \varepsilon_i Y_i X_i \tanh \parenth{Y_i X_i^\top \theta} \right\| \ge t/2},
    \end{align}
    where $\varepsilon_i$ are independent Rademacher random variables. We define a good event $\event := \{ \forall i \in [n], |Y_i| \le \tau, |X_i^\top \theta^*| \le C \tau \}$ as before, where $\tau = \Theta \parenth{\sqrt{\log (n/\delta)}}$. Then the probability defined in \eqref{eq:symmetrization_uniform} can be decomposed as
    \begin{align*}
        \Prob \parenth{\sup \limits_{\enorm{\theta} \le r} \left\|\frac{1}{n} \sum_{i=1}^n \varepsilon_i Y_i X_i \tanh \parenth{Y_i X_i^\top \theta} \right\| \ge t/2 \biggr| \event} + P(\event^c).
    \end{align*}
    
    We are interested in bounding the following quantity for Chernoff bound: 
    \begin{align*}
        \Exs \brackets{\exp\parenth{\sup \limits_{\enorm{\theta} \le r} \frac{\lambda}{n} \left\| \sum_{i=1}^n \varepsilon_i Y_i X_i \tanh \parenth{Y_i X_i^\top \theta} \right\| } \biggr| \event},
    \end{align*}
    where we used Chernoff-Bound with some $\lambda > 0$ for the last inequality. We first go some steps before we can apply the Ledoux-Talagrand contraction arguments \cite{Ledoux_Talagrand_1991}, with $f_i(\theta) := \tanh \parenth{|Y_i| X_i^\top \theta}$. First, we use discretization argument for removing $l_2$ norm inside the expectation. 
    \begin{align*}
        \Exs &\brackets{\exp\parenth{\sup \limits_{\enorm{\theta} \le r} \frac{\lambda}{n} \left\| \sum_{i=1}^n \varepsilon_i Y_i X_i \tanh \parenth{Y_i X_i^\top \theta} \right\| } \biggr| \event } \\
        &\le \Exs \brackets{\exp \parenth{\sup_{u \in \sphere^d} \sup \limits_{\enorm{\theta} \le r} \frac{\lambda}{n} \sum_{i=1}^n \varepsilon_i Y_i (X_i^\top u) \tanh \parenth{Y_i X_i^\top \theta} } \biggr| \event } \\
        &\le \Exs \brackets{\exp \parenth{\sup_{j \in [M]} \sup \limits_{\enorm{\theta} \le r} \frac{2\lambda}{n} \sum_{i=1}^n \varepsilon_i Y_i (X_i^\top u_j) \tanh \parenth{Y_i X_i^\top \theta} } \biggr| \event } \\
        &\le \sum_{j=1}^M \Exs \brackets{\exp \parenth{\sup \limits_{\enorm{\theta} \le r} \frac{2\lambda}{n} \sum_{i=1}^n \varepsilon_i Y_i (X_i^\top u_j) \tanh \parenth{Y_i X_i^\top \theta} } \biggr| \event },
    \end{align*}
    where $M$ is 1/2-covering number of the unit sphere and $\{u_1, ..., u_M\}$ is the corresponding covering set. Now for each $u_j$, we can apply the Ledoux-Talagrand contraction lemma since $|f_i(\theta_1) - f_i(\theta_2)| \le |Y_i| |X_i^\top \theta_1 - X_i^\top \theta_2|$ for $\theta \in \ball(0,r)$:
    \begin{align}
        \label{eq:uniform_mlr_ledoux_talagrand}
        \Exs &\brackets{\exp\parenth{\sup \limits_{\enorm{\theta} \le r} \frac{2\lambda}{n} \sum_{i=1}^n \varepsilon_i Y_i X_i^\top u_j \tanh \parenth{Y_i X_i^\top \theta} } \biggr| \event } \nonumber \\
        &= \Exs \brackets{\exp\parenth{\sup \limits_{\enorm{\theta} \le r} \frac{2\lambda}{n} \sum_{i=1}^n \varepsilon_i |Y_i| X_i^\top u_j \tanh \parenth{|Y_i| X_i^\top \theta} } \biggr| \event } \nonumber \\
        &\le \Exs \brackets{\exp \parenth{\sup \limits_{\enorm{\theta} \le r} \frac{2 \lambda}{n} \sum_{i=1}^n \varepsilon_i Y_i^2 (X_i^\top \theta)(X_i^\top u_j) } \biggr| \event } \nonumber \\
        &\le \Exs \brackets{\exp \parenth{\sup \limits_{\enorm{\theta} \le r} \frac{2 \lambda}{n} \sum_{i=1}^n \varepsilon_i Y_i^2 (X_i^\top v)(X_i^\top u_j) } \biggr| \event },
    \end{align}
    where we define $v := \theta / \enorm{\theta}$.
    
    We have already seen in \eqref{eq:conditioned_orcliz_norm} that $Y_i (X_i^\top u_j) | \event$ is sub-Gaussian with Orcliz norm $O(\tau (1 + \enorm{\theta^*})) = O(\tau)$. Since the multiplication of two sub-Gaussian variables is sub-exponential, it implies that $Y_i^2 (X_i^\top v) (X_i^\top u_1) | \event$ is sub-exponential with Orcliz norm $O(\tau^2)$ \cite{Vershynin_2011}. 
     Now we need the lemma for the exponential moment of sub-exponential random variables from \cite{Vershynin_2011}.
    \begin{lemma}[Lemma 5.15 in \cite{Vershynin_2011}]
        \label{lemma:subexp_expmoment_versynin}
        Let $X$ be a centered sub-exponential random variable. Then, for $t$ such that $t \le c / \vecnorm{X}{\psi_1}$, one has
        \begin{align*}
            \Exs[ \exp(tX) ] \le \exp(Ct^2 \vecnorm{X}{\psi_1}^2),
        \end{align*}
        for some universal constant $c, C > 0$. 
    \end{lemma}
    
    Finally, note that $\varepsilon_i Y_i^2 (X_i^\top v) (X_i^\top u_1)$ is a centered sub-exponential random variable with the same Orcliz norm. Equipped with the lemma, we can obtain that
    \begin{align*}
        \Exs \brackets{\exp \parenth{4 \lambda r \frac{1}{n} \sum_{i=1}^n \varepsilon_i Y_i^2 (X_i^\top v) (X_i^\top u_1)} \biggr| \event} \le \exp(C \lambda^2 r^2 \tau^4 / n), \qquad \forall |\lambda r / n| \le c / \tau^2,
    \end{align*}
    which yields
    \begin{align*}
        \Exs \brackets{\exp\parenth{\sup \limits_{\enorm{\theta} \le r} \frac{\lambda}{n} \left\| \sum_{i=1}^n \varepsilon_i Y_i X_i \tanh \parenth{Y_i X_i^\top \theta} \right\|} \biggr| \event } \le \exp \parenth{C \lambda^2 r^2 \tau^4 / n + C'd}, \ \forall |\lambda| \le n/c \tau^2 r,
    \end{align*}
    where we used $\log M = O(d)$ with some $C, C', c > 0$. Combining all the above, we have that
    \begin{align*}
        \Prob \biggr( \sup \limits_{\theta \in \ball(\theta^{*}, r)} \enorm{Z(\theta)} \ge t \biggr) \le \exp \parenth{C_0 \lambda^2 r^2 \tau^4 / n + C_1 d - \lambda t / 2} + \Prob(\event^c).
    \end{align*}
    From here, we can optimize for $\lambda = O(t / r^2 \tau^4)$ with setting $t = O \parenth{r \sqrt{d \tau^4 / n}}$. Since $t = O\parenth{r \sqrt{d \log^2(n/\delta)/n}}$, this concludes the proof. 
\end{proof}

\section{Supplementary Results}
\label{subsec:supplemnt_result}
In this appendix, we collect an additional result clarifying the initialization in Theorem~\ref{theorem:EM_mlr_upper_bound} and the proof for super-linear convergence of population EM operator in very high SNR regime. 
\subsection{Initialization with Spectral Methods}
\label{Appendix:2MLR_init_spectral}
\begin{lemma}
    \label{lemma:spectral_init}
    Let $M = \frac{1}{n} \sum_{i=1}^n Y_i^2 X_i X_i^\top - I$ where $X, Y$ are as given in Lemma \ref{lemma:concentration_subexp}. Let the largest eigenvalue and corresponding eigenvector of $M$ be $(\lambda_1, v_1)$. Then, there exists universal constants $c_0, c_1 > 0$ such that
    \begin{align*}
        |\lambda_1 - \enorm{\theta^*}^2| &\le c_0 (\enorm{\theta^*}^2 + 1) \sqrt{\frac{d \log^2 (n/\delta)}{n}}.
    \end{align*}
    Furthermore, if $\enorm{\theta^*} \ge c_1 (d \log^2(n/\delta) /n)^{1/4}$, then
    \begin{align*}
        \sin \angle(v_1, \theta^*) &\le c_0 \parenth{1 + \frac{1}{\enorm{\theta^*}^2}} \sqrt{\frac{d \log^2 (n/\delta)}{n}} \le \frac{1}{10}.
    \end{align*}
\end{lemma}
\begin{proof}
    The lemma is a direct consequence of Lemma \ref{lemma:concentration_higher_order} and matrix perturbation theory \cite{wainwright2019high}. Note that $\Exs[Y_i^2 X_i X_i^\top] = I + 2 \theta^* {\theta^*}^\top$ (e.g., see Lemma 1 in \cite{yi2016solving}).
\end{proof}
The above lemma states that when $\enorm{\theta^*}$ is not too small, we can always start from the well-initialized point where it is well aligned with ground truth $\theta^*$. In low SNR regime where $\enorm{\theta^*}^2 \lesssim (d/n)^{1/2}$, we cannot guarantee such a well-alignment with $\theta^*$ since the eigenvector is perturbed too much. However, the largest eigenvalue can still serve as an indicator that $\enorm{\theta^*}$ is small. Hence in all cases, we can initialize the estimator with $\theta_n^0 = \max \{ 0.2, \sqrt{\lambda_1} \} v_1$ to satisfy the initialization condition that we required in Theorem \ref{theorem:EM_mlr_upper_bound}. 


\subsection{Super-Linear Convergence of Population EM Operator in a Very High SNR Regime}
In this appendix, we prove Lemma \ref{lemma:super_linear_EM} on the super-linear convergence behavior of population EM operator in very high SNR regime.
\label{subsec:super_linear_EM}
\begin{proof}
    We start from the following equation:
    \begin{align*}
        \enorm{\popopmlr(\theta) - \theta^*} &= \Exs[XY (\tanh(YX^\top \theta)-\tanh(YX^\top \theta^*))] \\
        &= \Exs[XY \Delta_{(X,Y)} (\theta)],
    \end{align*}
    where $\Delta_{(X,Y)} (\theta) := \tanh(YX^\top \theta)-\tanh(YX^\top \theta^*)$. We define good events as follows:
    \begin{align}
        \label{eq:good_event}
        \event_{1} &= \{2 |X^\top (\theta^* - \theta)| \le |X^\top \theta^*|\}, \nonumber \\
        \event_{2} &= \{|X^\top \theta^*| \ge 2\tau \}, \nonumber \\
        \event_{3} &= \{|Z| \le \tau \}, 
    \end{align}
    where we set $\tau = \Theta \parenth{\sqrt{\log \enorm{\theta^*}}}$. 
    
    Let the good event $\event_{good} = \event_1 \cap \event_2 \cap \event_3$. From Lemma \ref{lemma:good_event_small_diff}, under the good event, we have $\Delta_{(X, Y)} (\theta) \le \exp(-\tau^2)$. 
    To simplify the notation, let $\Delta(\theta) = \Delta_{(X, Y)}(\theta)$ and $W = \nu X X^\top \theta^* \Delta(\theta)$. Then we can decompose the estimation error as the following:
    \begin{align*}
        \enorm{\popopmlr(\theta) - \theta^*} &= \enorm{\Exs[ X Z \Delta(\theta) ] + \Exs[W \Delta(\theta)]} \\
        &\le \sup_{u \in \sphere^{d-1}} |\Exs[(X^\top u) Z \Delta(\theta)]| + |\Exs[(W^\top u) \Delta(\theta)]| \\
        &\le \sup_{u \in \sphere^{d-1}} \sqrt{\Exs \brackets{(X^\top u)^2 |\Delta(\theta)|}} \sqrt{\Exs \brackets{Z^2 |\Delta(\theta)| }} \\
        &\quad + \sqrt{\Exs \brackets{(X^\top u)^2 |\Delta(\theta)|}} \sqrt{\Exs \brackets{(X^\top \theta^*)^2 |\Delta(\theta)|}}.
    \end{align*}
    We use again the event-wise decomposition strategy. For population EM, note that we set $\tau = \Theta(\sqrt{\log \enorm{\theta^*}})$ unlike in finite-sample EM case in Appendix \ref{subsec:full_proof_high_SNR}. We need to prove the following lemma:
    \begin{lemma}
        \label{lemma:simple_lemma_for_superlinear}
        For any $u \in \sphere^{d-1}$, we have
        \begin{align}
            \label{eq:event_decomp_unit_vec}
            \Exs \brackets{(X^\top u)^2 |\Delta(\theta)|} \le 4 \exp(-\tau^2 / 2) + 2(\tau + 2 \enorm{\theta - \theta^*}) / \enorm{\theta^*}.
        \end{align}
        Furthermore, we have
        \begin{align}
            \label{eq:event_decomp_theta}
            \Exs \brackets{(X^\top \theta^*)^2 |\Delta(\theta)|} \le 4 \enorm{\theta^*}^2 \exp(-\tau^2 / 2) + 8\tau^3 / \enorm{\theta^*} + 4 \enorm{\theta - \theta^*}^3 / \enorm{\theta^*}.
        \end{align}
        On the other hand, we have
        \begin{align}
            \label{eq:event_decomp_noise}
            \Exs \brackets{Z^2 |\Delta(\theta)|} \le 4 \exp(-\tau^2 / 4) + 2(\tau + \enorm{\theta - \theta^*}) / \enorm{\theta^*}.
        \end{align}
    \end{lemma}
    Equipped with the above lemma, whenever $\enorm{\theta - \theta^*} \ge C \tau$ with $\tau = c_2 \sqrt{\log \enorm{\theta^*}}$ for sufficiently large constants $C, c_2 > 0$, we have 
    \begin{align*}
        \Exs[(X^\top u)^2 |\Delta(\theta)|] &\le 5 \enorm{\theta - \theta^*} / \enorm{\theta^*}, \\
        \Exs[(X^\top \theta^*)^2 |\Delta(\theta)|] &\le 5 \enorm{\theta - \theta^*}^3 / \enorm{\theta^*}, \\
        \Exs[Z^2 |\Delta(\theta)|] &\le 5 \enorm{\theta - \theta^*} / \enorm{\theta^*},
    \end{align*}
    which yields $\enorm{\popopmlr(\theta) - \theta^*} \le 6 \enorm{\theta - \theta^*}^2 / \enorm{\theta^*}$, given that $\enorm{\theta^*}$ is sufficiently large and $\enorm{\theta - \theta^*} \le \enorm{\theta^*} / 10$.
\end{proof}

\paragraph{Proof of Lemma \ref{lemma:simple_lemma_for_superlinear}}
For equation~\eqref{eq:event_decomp_unit_vec}, we can check that 
\begin{align*}
    \Exs[(X^\top u)^2 |\Delta(\theta)|] &\le \Exs[(X^\top u)^2 |\Delta(\theta)| | \event_{good}] P(\event_{good}) + \Exs[(X^\top u)^2 |\Delta(\theta)| | \event_{1}^c] P(\event_1^c) \\
    &\quad + \Exs[(X^\top u)^2 |\Delta(\theta)| | \event_{2}^c] P(\event_2^c) + \Exs[(X^\top u)^2 |\Delta(\theta)| | \event_{3}^c] P(\event_3^c) \\
    &\le \exp(-\tau^2) \Exs[(X^\top u)^2 1_{\event_{good}}] + \Exs[(X^\top u)^2 | \event_{1}^c] P(\event_1^c) + \\
    &\quad + \Exs[(X^\top u)^2 | \event_{2}^c] P(\event_2^c) + \Exs[(X^\top u)^2 |\event_{3}^c] P(\event_3^c).
\end{align*}

We now recall Lemma 1 in \cite{yi2014alternating}, which is given by:
\begin{lemma}[Lemma 1 in \cite{yi2014alternating}]
\label{lemma:super_linear_lemma}
    Given vectors $u, v \in \mathbb{R}^d$ and a Gaussian random vector $X \sim \NORMAL(0, I)$, the matrix $\Sigma = \Exs[XX^\top | (X^\top u)^2 > (X^\top v)^2]$ has singular values
    \begin{align*}
        \parenth{1 + \frac{\sin \alpha}{\alpha}, 1 - \frac{\sin \alpha}{\alpha}, 1, 1, ..., 1}, \qquad \qquad \text{where } \quad \alpha = \cos^{-1} \parenth{\frac{(u - v)^\top (u+v)}{\enorm{u-v}\enorm{u+v}}}.
    \end{align*}
    Furthermore, if $\enorm{v} \le \enorm{u}$, then we have
    \begin{align*}
        P((X^\top u)^2 > (X^\top v)^2) \le \frac{\enorm{v}}{\enorm{u}}.
    \end{align*}
\end{lemma}
Based on the results of Lemma~\ref{lemma:super_linear_lemma}, we obtain 
\begin{align*}
    &\opnorm{\Exs[XX^\top | \event_1^c]} \le 2, \quad P(\event_1^c) \le 2 \enorm{\theta - \theta^*} / \enorm{\theta^*}.
\end{align*}

From standard property of Gaussian distribution, (see also Lemma 9 in \cite{Siva_2017}), we also have
\begin{align*}
    &\opnorm{\Exs[XX^\top | \event_2^c]} \le 1, \quad P(\event_2^c) \le 2 \tau / \enorm{\theta^*}. 
\end{align*}
Finally, from standard Gaussian tail bound, $P(\event_3^c) \le 2 \exp(-\tau^2/2)$. Plugging these relations, we get equation~\eqref{eq:event_decomp_unit_vec}. 

Similarly, we can check that
\begin{align*}
    \Exs[(X^\top u)^2 |\Delta(\theta)|] &\le \exp(-\tau^2) \Exs[(X^\top \theta^*)^2 1_{\event_{good}}] + \Exs[(X^\top \theta^*)^2 | \event_{1}^c] P(\event_1^c) + \\
    &\quad + \Exs[(X^\top \theta^*)^2 | \event_{2}^c] P(\event_2^c) + \Exs[(X^\top \theta^*)^2 |\event_{3}^c] P(\event_3^c) \\
    &\le \exp(-\tau^2) \Exs[(X^\top \theta^*)^2] + \Exs[(X^\top (\theta^* - \theta))^2 | \event_{1}^c] P(\event_1^c) + \\
    &\quad + 4 \Exs[\tau^2 | \event_{2}^c] P(\event_2^c) + \Exs[(X^\top \theta^*)^2 |\event_{3}^c] P(\event_3^c) \\
    &\le \exp(-\tau^2) \enorm{\theta^*}^2 + 4 \enorm{\theta^* - \theta}^3 / \enorm{\theta^*} + 8 \tau^3 / \enorm{\theta^*} + 2\enorm{\theta^*}^2 \exp (-\tau^2 / 2),
\end{align*}
which gives equation~\eqref{eq:event_decomp_theta}. 

Finally, for equation~\eqref{eq:event_decomp_noise}, 
\begin{align*}
    \Exs[Z^2 |\Delta(\theta)|] &\le \exp(-\tau^2) \Exs[Z^2 1_{\event_{good}}] + \Exs[Z^2 | \event_{1}^c] P(\event_1^c) + \Exs[Z^2 | \event_{2}^c] P(\event_2^c) + \Exs[Z^2 1_{\event_{3}^c}] \\
    &\le \exp(-\tau^2) + \Exs[Z^2] P(\event_1^c) + \Exs[Z^2] P(\event_2^c) + \sqrt{\Exs[Z^2]}\sqrt{P (\event_3^c)} \\
    &\le 4 \exp(-\tau^2/4) + 2 \tau / \enorm{\theta^*} + 2 \enorm{\theta - \theta^*} / \enorm{\theta^*},
\end{align*}
where we used the independence between $Z$ and $\event_1, \event_2$. This concludes the proof of Lemma \ref{lemma:simple_lemma_for_superlinear}. \hfill{$\square$}

\subsection{Proof of Theorem \ref{theorem:algebraic_independent}}
\label{appendix:proof:contraction_ind}

\subsubsection{Convergence in the Population Level}
\label{subsec:proof:contraction_pop_ind}

Given the EM updates of location and variance in equation~\eqref{eq:EM_updates_mixed_regression_unknown_noise}, the population version of the EM operation is given as follows:
\begin{align}
    \popopind(\theta) :=  \Exs_{(X, Y)} \brackets{XY \text{tanh} \parenth{\frac{Y X^{\top} \theta}{1 + \enorm{\theta^*}^2 - \enorm{ \theta}^2 }}}, \label{eq:pop_operator_indep}
\end{align}

    We recall some notations we defined in the beginning of the section. $\{v_1, ..., v_d\}$ is the standard basis in the transformed coordinate such that $v_1 = \theta / \enorm{\theta}$, and $\textbf{span}(v_1, v_2) = \textbf{span}(\theta, \theta^*)$. Let $x_1, x_2$ be $X^\top v_1, X^\top v_2$ respectively. Furthermore, denote $b_1^* = {\theta^*}^\top v_1 = \enorm{\theta^*} \cos \angle(\theta, \theta^*)$, and $b_2^* = {\theta^*}^\top v_2 = \enorm{\theta^*} \sin \angle(\theta, \theta^*)$. We will denote $\Delta = \|\theta^*\|^2 - \|\theta\|^2$. 
    
    We can rewrite the form of population operator in equation~\eqref{eq:pop_operator_indep} as follows:
    \begin{align*}
        \popopind ( \theta) & = \Exs_{(X, Y)} \brackets{X Y \text{tanh} \parenth{\frac{Y X^\top \enorm{\theta}}{1 + \Delta}}} \\
        &= \Exs_{(x_1, x_2, y)} \brackets{y x_1 \text{tanh} \parenth{\frac{y x_1 \enorm{\theta}}{1 + \Delta}}} v_1 + \Exs_{(x_1, x_2, y)} \brackets{y x_2 \text{tanh} \parenth{\frac{y x_1 \enorm{\theta}}{1 + \Delta}}} v_2.
    \end{align*}
    In fact, this expression is equivalent to \eqref{eq:pop_mlr_transformed} by replacing $\enorm{\theta} \leftarrow \frac{\enorm{\theta}}{1+\Delta}$. Therefore we can use the equation \eqref{eq:pop_mlr_firstvalue_taylor} with replacing $\enorm{\theta}$ such that,
    \begin{align*}
        \popopind (\theta)^\top v_1 &\le \frac{\enorm{\theta}}{1+\Delta} \parenth{1 - \frac{3\enorm{\theta}^2}{(1 + \Delta)^2} + \frac{30\enorm{\theta}^4}{(1+\Delta)^4}} + c_1 \frac{\enorm{\theta}}{1 + \Delta} \enorm{\theta^*}^2, \\
        \popopind (\theta)^\top v_2 &\le c_2 \frac{\enorm{\theta}}{1 + \Delta} \enorm{\theta^*}^2, 
    \end{align*}
    for some absolute constants $c_1, c_2 > 0$. We will show that $\frac{3}{(1 + \Delta)^2} - \frac{30 \enorm{\theta}^2}{(1+\Delta)^4} \ge 1.25$ whenever $\enorm{\theta} < 0.2$. Then we can conclude that $|\popopind (\theta)^\top v_1| \le \enorm{\theta}(1 - 0.25 \enorm{\theta}^2 + O(\enorm{\theta^*}^2))$.
    
    Now it is easy to check that
    \begin{align*}
        \frac{3}{(1 + \Delta)^2} - \frac{30 \enorm{\theta}^2}{(1 + \Delta)^4} &= \frac{3(1 + \Delta)^2 - 30\enorm{\theta}^2}{(1 + \Delta)^4} = \frac{3 - 36\enorm{\theta}^2 + 6\enorm{\theta^*}^2 + 3 \Delta^2}{(1 + \Delta)^4}.
    \end{align*}
    If $\enorm{\theta} < 0.2$, then $|\Delta| < 0.04$, $3 - 36\enorm{\theta}^2 \ge 1.5$ and $(1 + \Delta)^4 \le 1.16$, giving the desired bound for the first coordinate. Note that the second coordinate is already less than $O(\enorm{\theta} \enorm{\theta^*})$.

    We can also check that this is the best speed at which EM can converge. Observe that 
    \begin{align*}
        \enorm{\popopind(\theta)} \ge |\popopind(\theta)^\top v_1| &\ge \frac{\enorm{\theta}}{1+\Delta} \parenth{1 - \frac{3\enorm{\theta}^2}{(1+\Delta)^2}} - c_3 \frac{\enorm{\theta}}{1+\Delta} \enorm{\theta^*}^2 \\
        &\ge \enorm{\theta} \parenth{1 - 4\enorm{\theta}^2 - c_4 \enorm{\theta^*}^2},
    \end{align*}
    for some absolute constants $c_3, c_4 > 0$ where we simplify the coefficients using $\enorm{\theta} < 0.2$. Together with the upper bound we can conclude that
    \begin{align}
        \label{eq:em_converge_pop_ind}
        \enorm{\theta} (1 - 4\enorm{\theta}^2 - c_l \enorm{\theta^*}^2) \le \enorm{\popopind(\theta)} \le \enorm{\theta} (1 - 0.25 \enorm{\theta}^2 + c_u \enorm{\theta^*}^2), 
    \end{align}
    for some absolute constants $c_l, c_u > 0$, completing the proof.

\subsubsection{Uniform Deviations of Finite-Sample EM Operators}
\label{subsec:proof:lemma:deviation_bound_ind}

    Note that we assume $n \gtrsim d \ln^2 (n/\delta) / \epsilon^2$ for sufficiently small $\epsilon > 0$. To simplify the notation, we use $\hat{\Sigma}_n = \frac{1}{n}\sum_i X_iX_i^\top$ and $\bar{\sigma}_n^2 = \frac{1}{n}\sum_i Y_i^2 - \frac{1}{n} \sum_i (X_i^\top \theta)^2$. 
    We also define $$\widetilde{M}_{ind}(\theta) := (\sum_{i=1}^nX_i X_i^\top)^{-1} \sum_{i=1}^n Y_i X_i \tanh \parenth{\frac{Y_i X_i^\top \theta}{1+\Delta}}.$$ Then we can see that
    \begin{align*}
        \enorm{\samopind (\theta) - \popopind (\theta)} &\le \left\| \samopind (\theta) - \widetilde{M}_{ind}(\theta) \right\| + \left\| \widetilde{M}_{ind}(\theta) - \popopind (\theta)  \right\| \\
        &\le \|\hat{\Sigma}_n^{-1}\| \underbrace{\left\| \frac{1}{n} \sum_{i} X_i Y_i \parenth{\tanh \parenth{\frac{Y_iX_i^\top \theta}{\sigma_n^2}} - \tanh \parenth{\frac{Y_iX_i^\top \theta}{1 +\Delta} }} \right\|}_{(a)} \\
        &+ \|\hat{\Sigma}_n^{-1}\| \underbrace{\left\| {\frac{1}{n} \sum_{i} X_i Y_i \tanh \parenth{\frac{Y_iX_i^\top \theta}{1 +\Delta }} } - \Exs \brackets{XY \tanh \parenth{ \frac{Y X^\top \theta}{1 +\Delta} }} \right\|}_{(b)} \\
        &+ \underbrace{\|\hat{\Sigma}_n^{-1} - I\| \left\| \Exs \brackets{XY \tanh \parenth{ \frac{Y X^\top \theta}{1 +\Delta} }} \right\|}_{(c)}.
    \end{align*}
    For bounding (a), we first note that by the concentration lemmas, we have $\bar{\sigma}_n^2 \approx 1 + \Delta + O(\epsilon)$. It is also easy to verify that $|\tanh(a) - \tanh(b)| \le |a-b|$ for any $a,b \in \mathbb{R}$. Now for any unit vector $u \in \sphere^d$,  
    \begin{align*}
        \frac{1}{n} \sum_{i} (X_i^\top u) Y_i &\parenth{\tanh \parenth{\frac{Y_iX_i^\top \theta}{\bar{\sigma}_n^2}} - \tanh \parenth{\frac{Y_iX_i^\top \theta}{1 + \Delta}}} \\
        &\le \frac{1}{n} \sqrt{\sum_{i} (X_i^\top u)^2 Y_i^2} \sqrt{\sum_i \parenth{ \frac{Y_iX_i^\top \theta}{\bar{\sigma}_n^2} - \frac{Y_iX_i^\top \theta}{(1 + \Delta)}}^2} \\
        &\le \frac{1}{n} \sqrt{\opnorm{\sum_i Y_i^2 X_i X_i^\top}} \sqrt{\sum_{i} \epsilon^2 Y_i^2 \frac{(X_i^\top \theta)^2}{(1 + \Delta)^2}} \\
        &\le \frac{\epsilon \enorm{\theta}}{1 + \Delta} \opnorm{\frac{1}{n} \sum_i Y_i^2 X_i X_i^\top} \le 2\epsilon \enorm{\theta},
    \end{align*}
    Finally, we can use Lemma \ref{lemma:concentration_higher_order} to get $(a) \le O(\epsilon \enorm{\theta})$.
    
    For the left two terms, (b) is bounded with applying the Lemma \ref{lemma:uniform_concentration_mlr} by plugging $\theta \leftarrow \theta / (1 + \Delta)$. (c) is bounded by the concentration of $\hat{\Sigma}_n$ in Lemma \ref{lemma:concentration_subexp} and the fact $\enorm{\popopind(\theta)} \le \enorm{\theta}$ from \eqref{eq:em_converge_pop_ind}. The rest of the steps follow the same argument as in the case for known variances (see Appendix \ref{subsec:full_proof_low_SNR}). This conclude the Theorem \ref{theorem:algebraic_independent}.